%% file: arxiv_v1_main.tex
\documentclass[11pt]{article}

\usepackage{jmlr2e}

\usepackage{microtype}
\usepackage{graphicx}
\usepackage{subfigure}
\usepackage{booktabs} 
\usepackage{multirow}
\usepackage{amssymb}
\usepackage{mathtools}

\usepackage{enumitem}

\usepackage{amsmath}

\usepackage{algorithm}
\usepackage{algorithmic}

\newtheorem{assumption}{Assumption}

\newcommand{\blue}[1]{\textcolor{blue}{#1}}

\def \A {\mathbf{A}}

\def \R {\mathbb{R}}

\def \vbf {\mathbf{v}}

\def \x {\mathbf{x}}

\def \E {\mathbb{E}}

\def \a {\mathbf{a}}

\def \dbf {\mathbf{d}}
\def \z {\mathbf{z}}

\def \y {\mathbf{y}}

\def \K {\mathcal{K}}

\def \g {\mathbf{g}}
\def \F {\mathcal{F}}

\usepackage{hyperref}
\usepackage[capitalize,noabbrev]{cleveref}
\hypersetup{colorlinks={true},linkcolor={magenta},citecolor={blue}}

\begin{document}

\title{Communication-Efficient Decentralized Online Continuous DR-Submodular Maximization}

\author{\name Qixin Zhang \email qxzhang4-c@my.cityu.edu.hk\\
       \addr School of Data Science\\
       City University of Hong Kong\\
       Kowloon, Hong Kong, China
      \AND
       \name Zengde Deng \email zengde.dzd@cainiao.com \\
      \addr Cainiao Network\\
      Hang Zhou, China
       \AND
       \name Xiangru Jian \email xiangjian2-c@my.cityu.edu.hk\\
      \addr School of Data Science\\
       City University of Hong Kong\\
       Kowloon, Hong Kong, China
       \AND
       \name Zaiyi Chen \email zaiyi.czy@cainiao.com \\
      \addr Cainiao Network\\
      Hang Zhou, China
      \AND
       \name Haoyuan Hu \email haoyuan.huhy@cainiao.com \\
      \addr Cainiao Network\\
      Hang Zhou, China
      \AND
      \name Yu Yang \email yuyang@cityu.edu.hk\\
      \addr School of Data Science\\
       City University of Hong Kong\\
       Kowloon, Hong Kong, China
      }

\maketitle

\begin{abstract}
Maximizing a monotone submodular function is a fundamental task in machine learning, economics, and statistics. In this paper,  we present two communication-efficient decentralized online algorithms for the monotone continuous DR-submodular maximization problem, both of which reduce the number of per-function gradient evaluations and per-round communication complexity from $T^{3/2}$ to $1$. The first one, One-shot Decentralized Meta-Frank-Wolfe~(\textbf{Mono-DMFW}), achieves a $(1-1/e)$-regret bound of $O(T^{4/5})$. As far as we know, this is the first \emph{one-shot} and \emph{projection-free} decentralized online algorithm for monotone continuous DR-submodular maximization. Next, inspired by the non-oblivious boosting function~\citep{zhang2022boosting}, we propose the Decentralized Online Boosting Gradient Ascent~(\textbf{DOBGA}) algorithm, which attains a $(1-1/e)$-regret of $O(\sqrt{T})$. To the best of our knowledge, this is the first result to obtain the optimal $O(\sqrt{T})$ against a $(1-1/e)$-approximation with only one gradient inquiry for each local objective function per step. Finally, various experimental results confirm the effectiveness of the proposed methods.
\end{abstract}

\section{Introduction}
\input{arxiv_v1/Introduction}

\section{Related Works}
\input{arxiv_v1/Related_Works}

\section{Preliminaries}
\input{arxiv_v1/Preliminaries}

\section{One-Shot Decentralized Meta-Frank-Wolfe Algorithm}\label{sec:one-shot}
\input{arxiv_v1/One_shot_Frank_wolfe}

\section{Decentralized Online Boosting Gradient Ascent}\label{sec:boosting}
\input{arxiv_v1/One_shot_gradient_ascent}

\section{Numerical Experiments}
\input{arxiv_v1/Experiment}

\section{Conclusion}
In this paper, we propose two communication-efficient decentralized online algorithms for optimizing the monotone continuous DR-submodular maximization problem over the network, both of which improve the communication complexity and the number of per-function gradient evaluations from $T^{3/2}$ to $1$. First, we develop the \textbf{Mono-MFW} algorithm, achieving a $(1-1/e)$-regret bound  of $O(T^{4/5})$.  Then, we present a variant of the decentralized online gradient ascent, namely, \textbf{DOBGA} algorithm, which attains a $(1-1/e)$-regret of $O(\sqrt{T})$. Numerical experiments demonstrate the effectiveness of the proposed algorithms.
\vskip 0.2in

\bibliography{references}
\appendix

\section{Proofs in Section~\ref{sec:one-shot}}\label{appendix:one-shot}
\input{arxiv_v1/Appendix1}

\section{Proofs in Section~\ref{sec:boosting}}\label{appendix:boosting}
\input{arxiv_v1/Appendix2}

\end{document}

%% file: arxiv_v1/Introduction.tex
With the rapid development of digital systems, communication, and sensing technologies~\citep{rabbat2004distributed,abu2013data}, 
numerous large-scale datasets are collected over the networked machines. 
To deal with large-scale datasets in data-intensive applications, it is urgent to design algorithms in a decentralized manner, where computing units cooperatively optimize the global objective functions throughout local computation and network communications among each other. Clearly, compared to centralized algorithms, decentralized algorithms that efficiently exploiting dispersed computational resources are more scalable. Moreover, decentralized algorithms have advantages in relieving the data privacy risk, as computing nodes often only share very limited local information with each other.


Continuous DR-submodular maximization, an important subclass of the non-convex/non-concave optimization, has received considerable attention in recent years, due to its various applications in machine learning, economics, and statistics. For example, viral marketing~\citep{kempe2003maximizing,yang2016continuous}, revenue maximization~\citep{soma2017non,bian2020continuous}, non-definite quadratic programming~\citep{ito2016large}, determinantal point processes~\citep{kulesza2012determinantal,mitra2021submodular+}, and so on. Previously, a myriad of the existing literature has studied the \emph{centralized} and \emph{static} monotone continuous DR-submodular maximization problems~\citep{bian2017guaranteed,hassani2017gradient,mokhtari2018conditional,hassani2020stochastic}. However, in many real scenarios, the objectives not only are stored in a network of computing nodes but also vary with time~\citep{hazan2016introduction}. Thus, in this paper, we focus on the decentralized online monotone continuous DR-submodular maximization, where only the stochastic gradient oracles are available. 

Recently, \citet{zhu2021projection} have proposed the first decentralized online algorithm~(i.e., the decentralized Meta-Frank-Wolfe~(\textbf{DMFW})) for monotone continuous DR-submodular maximization, which achieves $(1-1/e)$-approximation guarantee with an expected regret bound of $O(\sqrt{T})$ where $T$ is the time horizon. 
Noticeably, at each round, \textbf{DMFW} needs to inquire $T^{3/2}$ stochastic gradient estimates for each local objective function and then passes these gradient messages over the network one by one, resulting in a large computation and communication overhead when $T$ is large.
In view of this, we consider the following question:
\begin{enumerate}
	\item[] \textbf{Can we design communication-efficient algorithms for the decentralized online monotone continuous DR-submodular maximization problem which guarantee the tight $(1-1/e)$ approximation ratio?}
\end{enumerate}

This paper provides an affirmative answer to this question by presenting two communication-efficient decentralized online algorithms for monotone continuous DR-submodular maximization, which both reduce the per-round communication complexity and the gradient evaluations of each local objective from $T^{3/2}$ to $1$. The first one, \textbf{Mono-DMFW} algorithm, equipped with the blocking procedures~\citep{zhang2019online}, attains a $(1-1/e)$-regret of $O(T^{4/5})$. Then, motivated via the non-oblivious function in \citet{zhang2022boosting}, we propose the decentralized online boosting gradient ascent~(\textbf{DOBGA}) with $(1-1/e)$-regret of $O(\sqrt{T})$. The contributions of this paper are summarized as follows.
\begin{itemize}
	\item Inspired by the blocking techniques in \citet{zhang2019online}, we first propose the \textbf{Mono-DMFW} algorithm, which improves the number of per-function gradient evaluations and per-round communication complexity from $T^{3/2}$ to $1$. Moreover, the \textbf{Mono-DMFW} algorithm achieves an expected regret of $O(T^{4/5})$ against the tight $(1-1/e)$-approximation to the best feasible solution in hindsight. To the best of our knowledge, the \textbf{Mono-DMFW} is the first \emph{one-shot} and \emph{projection-free} algorithm for decentralized online monotone continuous DR-submodular maximization.
	\item Next, we present the decentralized online boosting gradient ascent (\textbf{DOBGA}), which merges
	the non-oblivious auxiliary function~\citep{zhang2022boosting} into the classical \emph{one-shot} distributed online gradient ascent framework~\citep{yan2012distributed}. We proceed to verify that the \textbf{DOBGA} attains a $(1-1/e)$-regret of $O(\sqrt{T})$. It is worth mentioning that the $(1-1/e)$-regret of $O(\sqrt{T})$ result not only achieves the tight approximation guarantee for the monotone continuous DR-submodular maximization~\citep{bian2017guaranteed}, but also matches the optimal $O(\sqrt{T})$ regret~\citep{hazan2016introduction}.
	\item Finally, we evaluate the performance of our proposed algorithms on the real-world dataset. Numerical experiments demonstrate the effectiveness of our methods.
\end{itemize}

%% file: arxiv_v1/Related_Works.tex
\begin{table}[t]
	\renewcommand\arraystretch{1.35}
	\centering
	\caption{Comparison of decentralized algorithms for online continuous monotone DR-submodular function maximization with stochastic gradient oracles. '\textbf{Ratio}' means approximation ratio; '\textbf{\# Com}' means the communication complexity for each node at each round; '\textbf{\# Grad}' means the number of stochastic gradient evaluations for each local objective function at each round, and '\textbf{Projection-free}' indicates whether the algorithm is projection-free.}
    \resizebox{\textwidth}{!}{
    \begin{tabular}{c|c|c|c|c|c}
        \toprule[1.5pt]
        \textbf{Method} &\textbf{Ratio}& \textbf{Regret}& \textbf{\# Com}&\textbf{\# Grad}&\textbf{Projection-free}\\
        \hline
        \textbf{DMFW}~\citep{zhu2021projection}&$1-1/e$&$O(\sqrt{T})$&$T^{3/2}$&$T^{3/2}$&Yes\\
        \hline 
        \textbf{Mono-DMFW}~(This paper)&$1-1/e$&$O(T^{4/5})$&$1$&$1$&Yes\\
        \hline
        \textbf{DOBGA}~(This paper)&$1-1/e$&$O(\sqrt{T})$&$1$&$1$&No\\
        \midrule[1.5pt]
    \end{tabular}
    }
	\label{tab:compare}
\end{table} 

Continuous DR-submodular maximization has been extensively investigated since it admits efficient approximate maximization routines.
First, in the deterministic setting, \citet{bian2017guaranteed} presented a variant of Frank-Wolfe for maximizing \emph{centralized} and \emph{static} monotone continuous DR-submodular function, which achieves the tight $(1-1/e)$-approximation guarantee. \citet{hassani2017gradient} proceeded to show that the canonical stochastic gradient ascent achieves $(1/2)$-approximation guarantee. Next, \citet{mokhtari2018conditional} proposed the stochastic continuous greedy algorithm, which attains a $(1-1/e)OPT-\epsilon$ after $O(1/\epsilon^{3})$ iterations where $OPT$ is the optimal value. Then, an accelerated stochastic continuous greedy algorithm is presented in \citet{hassani2020stochastic}, which guarantees a $(1-1/e)$-approximation after $O(1/\epsilon^{2})$ iterations. In the online setting, \citet{chen2018projection,chen2018online} first proposed a novel Meta-Frank-Wolfe algorithm achieving the tight $(1-1/e)$-approximation of square-root regret under both deterministic and stochastic setting. After that, \citet{zhang2019online} proposed the blocking techniques to improve the per-function gradient evaluations of Meta-Frank-Wolfe from $T^{3/2}$~\citep{chen2018projection} and $T^{1/2}$~\citep{chen2018online} to $1$. Like the offline gradient ascent algorithm~\citep{hassani2017gradient}, the classical online gradient ascent also attains the suboptimal $(1/2)$-approximation guarantee~\citep{chen2018online}. Recently, \citet{zhang2022boosting} have proposed an auxiliary function to boost the approximation ratio of the offline and online gradient ascent algorithms from $1/2$ to $1-1/e$.

A large body of literature has studied decentralized convex optimization problems. \citet{nedic2009distributed} first proposed the decentralized gradient descent algorithm, which combines the consensus technique with local gradient descent. \citet{yuan2016convergence} proceeded to derive the convergence of decentralized gradient descent. However, these works assume that the objective function is unchanged with time. To tackle the varying objectives, \citet{yan2012distributed} then proposed the framework of decentralized online optimization, where they proved that the decentralized online projected gradient descent algorithm achieves a regret of $O(\log(T))$ and $O(\sqrt{T})$ for strongly convex objectives and convex function, respectively. Next, \citet{hosseini2013online} presented the distributed online dual averaging algorithm with the regret of $O(\sqrt{T})$ for convex objective. Inspired by the online conditional gradient~\citep{hazan2016introduction}, \citet{zwp2017projection} proposed a projection-free decentralized online algorithm for convex optimization and proved a regret bound of $O(T^{3/4})$. As for the decentralized submodular maximization, in the offline setting, \citet{mokhtari2018decentralized} proposed a decentralized continuous greedy algorithm with the tight $(1-1/e)$-approximation guarantee. Next, \citet{gao2021sample} presented a sample efficient decentralized algorithm for continuous DR-submodular maximization. Under the online scenario, \citet{zhu2021projection} merged the consensus technique and variance reduction technique into the Meta-Frank-Wolfe~\citep{chen2018online} to propose the decentralized Meta-Frank-Wolfe (\textbf{DMFW}) with an expected $(1-1/e)$-regret of $O(\sqrt{T})$, which needs to inquire $T^{3/2}$ stochastic gradients for each local function and communicate these information over network.

We make a comparison between the state-of-the-art \textbf{DMFW} and our algorithms for decentralized online monotone continuous DR-submodular maximization in \cref{tab:compare}.

%% file: arxiv_v1/Preliminaries.tex
In this section, we introduce some necessary notions.

\subsection{Notations}
In this paper, we use $\R$ and $\R_{+}$ to denote the set of real numbers and non-negative real numbers, respectively. 
Besides, the lower boldface (e.g. $\x$) denotes a column vector with suitable dimension and uppercase boldface (e.g. $\A$) for a matrix. $\x^{T}$ and $\A^{T}$ denote the transpose of the vector $\x$ and the matrix $\A$, respectively. Specially, $\mathbf{I}$ and $\mathbf{1}$ represent the identity matrix and the vector whose all entries are $1$, respectively. We denote the $i$-th element of a vector $\x$ as $(\x)_{i}$. Given two $n$-dimensional vector $\x$ and $\y$, we say $\x\ge\y$ if $(\x)_{i}\ge(\y)_{i}$ for all $i\in[n]$. The product $\langle\x,\y\rangle=\sum_{i}(\x)_{i}(\y)_{i}$ and the norm $\left\|\x\right\|=\sqrt{\langle\x,\x\rangle}$. Additionally, for any convex set $\mathcal{K}$, the radius $r(\mathcal{K})=\max_{\x\in\K} \left\|\x\right\|$ and the diameter $\mathrm{diam}(\K)=\max_{\x,\y\in\K} \left\|\x-\y\right\|$. For any positive integer $K$, $[K]$ denotes the set $\{1,\dots, K\}$. The symbol $\otimes$ denotes the Kronecker product.
\subsection{Continuous DR-Submodularity}
In this subsection, we recall the definition of continuous DR-submodular function. A continuous function $f:\mathcal{X}\rightarrow\R_{+}$ is {\it DR-submodular} if
\begin{align*}
	f(\x+z\boldsymbol{e}_{i})-f(\x)\le f(\y+z\boldsymbol{e}_{i})-f(\y),
\end{align*} where $\boldsymbol{e}_{i}$ is the $i$-th basic vector, $\x\ge\y$ and $z\in \mathbb{R}_{+}$ such that $\x+z\boldsymbol{e}_{i}, \y+z\boldsymbol{e}_{i}\in\mathcal{X}$.
Here, $\mathcal{X}=\prod_{i=1}^{n}\mathcal{X}_{i}$ where each $\mathcal{X}_{i}$ is a compact interval in $\mathbb{R}_{+}$. In this paper, we assume $\mathcal{X}_{i}=[0,a_{i}]$ where $a_{i}\in\R_{+}$. If the DR-submodular function $f$ is differentiable, we have $\nabla f(\x)\le\nabla f(\y)$ for any $\x\ge\y$. Furthermore, the DR-submodularity of function $f$ implies that $f$ is concave along the positive direction~\citep{bian2020continuous}, i.e., for any $\x,\y\in\mathcal{X}$ and $\y\ge\x$, it holds that 
\begin{equation*}
   f(\y)\le f(\x)+\langle\nabla f(\x), \y-\x\rangle. 
\end{equation*}
If $f$ is twice differentiable, the continuous DR-submodularity is equivalent to 
\begin{align*}
    \forall i,j\in[n],\forall\x\in\mathcal{X},\frac{\partial^{2}f(\x)}{\partial x_{i}\partial x_{j}}\le 0.
\end{align*}
Moreover, a differentiable function $f:\mathcal{X}\rightarrow\R_{+}$ is called $L$-$smooth$ if for any $\x,\y\in\mathcal{X}$, we have 
\begin{equation*}
    f(\y)\le f(\x)+\langle\nabla f(\x), \y-\x\rangle+\frac{L}{2}\|\y-\x\|^{2},
\end{equation*} which implies that the gradient of $f$ is $L$-lipschitz, i.e., $\left\|\nabla f(\x)-\nabla f(\y)\right\|\le L\left\|\x-\y\right\|$. Finally, we say that the function $f$ is {\it monotone} iff $f(\x)\ge f(\y)$ for any  $\x\ge\y$ and $\x,\y\in\mathcal{X}$.

\subsection{Problem Formulation}
In this subsection, we introduce the decentralized online monotone continuous DR-submodular maximization problem. 
We use the undirected connected graph $\mathcal{G}=(\mathcal{V},\mathcal{W})$ to denote the network of the computing nodes, where $\mathcal{V}=\{1,\dots,N\}$ represents the set of nodes and $\mathcal{W}\subseteq\mathcal{V}\times\mathcal{V}$ denotes the set of edges. 
The symbol $\mathcal{N}_{i}:=\{j\in\mathcal{V}|(i,j)\in\mathcal{W}\}$ denotes the set of the neighbors of node $i$. 
In this paper, we assume that each node $i\in\mathcal{V}$ only gets access to its local objective function and communicates the information with its neighbors in $\mathcal{N}_{i}$. Also, we define $a_{ij}\ge0$ to be the weight that node $i$ assigns to node $j$. If $(i,j)\notin\mathcal{W}$, $a_{ij}=0$. Moreover, the weight matrix $\A=[a_{ij}]\in\R_{+}^{N\times N}$ satisfies the following assumption.
\begin{assumption}\label{ass1}
	The matrix $\A\in\R_{+}^{N\times N}$ is symmetric and doubly stochastic, i.e., $\A^{T}=\A$ and $\A\mathbf{1}=\mathbf{1}$. Regarding the eigenvalue of $\A$, i.e., $1=\lambda_{1}(\A)\ge\lambda_{2}(\A)\dots\ge\lambda_{n}(\A)\ge-1$, the $\beta<1$, where $\beta=\max(|\lambda_{2}(\A)|,|\lambda_{n}(\A)|)$ is the second largest magnitude of the eigenvalues of $\A$.
\end{assumption}

In a $T$-$round$ decentralized online optimization, each node $i$ first chooses an action $\x_{i}(t)$ from the constraint set $\mathcal{K}$ at each round. Then, the adversary reveals a continuous DR-submodular function $f_{t,i}:\mathcal{X}\rightarrow\R_{+}$ and feeds back the reward $f_{t,i}(\x_{i}(t))$ to the node $i$. The goal of nodes is to online maximize the aggregate continuous DR-submodular function, i.e., $\max_{\x\in\mathcal{K}}\sum_{t=1}^{T}\sum_{i=1}^{N}f_{t,i}(\x)$. Note that, according to \citet{bian2017guaranteed}, maximizing a monotone continuous DR-submodular function subject to a general convex constraint is NP-hard. As a result, we turn to the $\alpha$-regret of each node $j\in\mathcal{V}$ as follows :
\begin{equation}\label{equ:regret}
	\mathcal{R}_{\alpha}(T,j)=\alpha\sup_{\x\in\mathcal{K}}\frac{1}{N}\sum_{t=1}^{T}\sum_{i=1}^{N}f_{t,i}(\x)-\frac{1}{N}\sum_{t=1}^{T}\sum_{i=1}^{N}f_{t,i}(\x_{j}(t)),
\end{equation} 
where $\alpha$ is the approximation ratio. Therefore, in this paper, we aim to design communication-efficient decentralized online algorithms such that 1)~at each round, each node $i\in\mathcal{V}$ only passes messages to its neighbors $\mathcal{N}_{i}$, and 2)~the $\alpha$-regret of each node $i\in\mathcal{V}$ is sublinear in $T$, namely, $\lim_{T\rightarrow\infty}\max_{i\in\mathcal{V}} \mathcal{R}_{\alpha}(T,i)/T=0$. In the following sections, we will consider the tight approximation ratio $\alpha=1-1/e$~\citep{bian2017guaranteed}. Furthermore, we make the following assumptions throughout this paper.
\begin{assumption}\label{ass2}
	The domain $\mathcal{K}\subseteq\mathcal{X}$ is a bounded convex set.
\end{assumption}

\begin{assumption}\label{ass3}
	Each local function $f_{t,i}:\mathcal{X}\rightarrow\R_{+}$ is differentiable, monotone, continuous DR-submodular and $L$-smooth, where $i\in[N]$ and $t\in[T]$.
\end{assumption}

\begin{assumption}\label{ass4}
	For any $t\in[T]$ and $i\in[N]$, there exists a stochastic gradient oracle $\widetilde{\nabla}f_{t,i}(\x)$ with  $\E(\widetilde{\nabla}f_{t,i}(\x)|\x)=\nabla f_{t,i}(\x)$ and $\E(\|\nabla f_{t,i}(\x)-\widetilde{\nabla}f_{t,i}(\x)\|^{2})\le\sigma^{2}$.
\end{assumption}

%% file: arxiv_v1/One_shot_Frank_wolfe.tex
To begin, we review the detailed results of \textbf{DMFW} algorithm ~\citep{zhu2021projection}, which is the first algorithm for the decentralized online monotone continuous DR-submodular maximization problem. The \textbf{DMFW} needs to inquire stochastic gradients at $K$ different points for each local function and then sequentially passes these gradients information to the neighboring nodes, which will incur the $K$ amounts of communications over the network at each round. \citet{zhu2021projection} also have verified that the \textbf{DMFW} achieves a $(1-1/e)$-regret of $O(\sqrt{T}+\frac{T}{K^{1/3}})$. In order to attain the lowest $O(\sqrt{T})$-regret bound, we usually set the $K=T^{3/2}$, which will incur huge gradient evaluations and prohibitive communication complexity when $T$ is large. As a result, in this section, we propose the one-shot decentralized Meta-Frank-Wolfe algorithm~(\textbf{Mono-DMFW}), which reduces the number of both per-round communication complexity and per-function gradient evaluations from $K$ to $1$.

\begin{algorithm}[t]
	\caption{One-Shot Decentralized Meta-Frank-Wolfe~(\textbf{Mono-DMFW})}
	\begin{algorithmic}[1]\label{alg:1}
	\STATE{\bf Input:} Time horizon $T$, the number of nodes $N$, blocking parameter $Q$ and $K$ where $T=KQ$, weight matrix $\A=[a_{ij}]\in\R^{N\times N}$, step size $\eta_{k},\forall k\in[K]$, parameter $\gamma$.
	\STATE {\bf Output:} $\{\x_{i}(t): i\in[N],t\in[T]\}$.
	\STATE Initialize $K$ online linear optimization oracles, $\mathcal{E}_{i}^{(1)},\dots,\mathcal{E}_{i}^{(K)}$ for each $i\in[N]$.
	\STATE For any $q\in[Q]$, initialize $\dbf_{i}^{(0)}(q)=\x_{i}^{(0)}(q)=\g_{i}^{(0)}(q)=\mathbf{0}$.
    \FOR{$q=1,\ldots, Q$}
    \FOR{$k=1,\dots,K$}
    \FOR{$i=1,\dots,N$}
    \STATE Receive the update direction $\vbf_{i}^{(k)}(q)$ which is the output of oracle $\mathcal{E}_{i}^{(k)}$.
    \STATE $\x_{i}^{(k)}(q)=\sum_{j\in\mathcal{N}_{i}\cup\{i\}}a_{ij}\x_{j}^{(k-1)}(q)+\frac{1}{K}\vbf_{i}^{(k)}(q)$.
	\ENDFOR
	\ENDFOR
	\FOR{$i=1,\dots,N$}
	\FOR{$t=(q-1)K+1,\dots,qK$} 
	\STATE Node $i$ plays $\x_{i}(t)=\x_{i}^{(K)}(q)$ to get reward $f_{t,i}(\x_{i}(t))$ and observes the stochastic gradient information of $f_{t,i}$.
	\ENDFOR
	\STATE Generate a random permutation $\{t_{i}^{(1)}(q),\dots,t_{i}^{(K)}(q)\}$ for $\{(q-1)K+1,\dots,qK\}$.
	\FOR{$k=1,\dots,K$}
	\STATE $\g_{i}^{(k)}(q)=(1-\eta_{k})\g_{i}^{(k-1)}(q)+\eta_{k}\widetilde{\nabla}f_{t_{i}^{(k)}(q),i}(\x_{i}^{(k)}(q))$.
	\STATE $\dbf_{i}^{(k)}(q)=(1-\gamma)\sum_{j\in\mathcal{N}_{i}\cup\{i\}}a_{ij}\dbf_{j}^{(k-1)}(q)+\gamma\g_{i}^{(k)}(q)$.
	\STATE Feed back $\langle\dbf_{i}^{(k)}(q),\vbf_{i}^{(k)}(q)\rangle$ as the payoff to be received by oracle $\mathcal{E}_{i}^{(k)}$.
	\ENDFOR
	\ENDFOR
	\ENDFOR
	\end{algorithmic}
\end{algorithm}

To obtain the one-shot decentralized algorithm, we adopt the blocking procedure in \citet{zhang2019online}. For each node $i\in\mathcal{N}$, we divide the $T$ objective functions $f_{1,i},\dots,f_{T,i}$ into $Q$ equisized blocks of size $K$ where $T=QK$. For instance, $\{f_{(q-1)K+1,i},\dots,f_{qK,i}\}$ are included in the $q$-th block of node $i$. Then, we define the average function in the $q$-th block of node $i$ as $\bar{f}_{q,i}=\sum_{k=1}^{K}f_{(q-1)K+k,i}/K$ where $q\in [Q]$. If we view the $\bar{f}_{1,i},\dots,\bar{f}_{Q,i}$ as the virtual local objective functions for node $i$, we could obtain a new $Q$-round decentralized online problem. That is, after each node $i$ first chooses an action $\y_{i}(q)\in\mathcal{K}$ at each round, the environment reveals the reward function $\bar{f}_{q,i}$ and feeds back the reward $\bar{f}_{q,i}(\y_{i}(q))$ to the node $i$. The goal of us is also to minimize the $\alpha$-regret $\bar{\mathcal{R}}_{\alpha}(Q,i)$ where $\bar{\mathcal{R}}_{\alpha}(Q,i)=\alpha\sup_{\x\in\mathcal{K}}\frac{1}{N}\sum_{q=1}^{Q}\sum_{j=1}^{N}\bar{f}_{q,j}(\x)-\frac{1}{N}\sum_{q=1}^{K}\sum_{j=1}^{N}\bar{f}_{q,j}(\y_{i}(q))$.

Note that, if we play the same action $\y_{i}(q)\in\mathcal{K}$ for the all upcoming objective functions in the $q$-th block of node $i$, the $\alpha$-regret of the original $T$-round decentralized online problem is $K$ times as that of the new $Q$-round game, i.e., $\mathcal{R}_{\alpha}(T,i)=K\bar{\mathcal{R}}_{\alpha}(Q,i)$. Therefore, we have $\mathcal{R}_{\alpha}(T,i)/T=K\bar{\mathcal{R}}_{\alpha}(Q,i)/T=\bar{\mathcal{R}}_{\alpha}(Q,i)/Q$, meaning that an algorithm with sublinear regret for the new $Q$-round problem could provide a feasible solution with sublinear regret for the original $T$-round game. Since each local average function $\bar{f}_{q,i}$ is also monotone and continuous DR-submodular, a simple way to design a sublinear algorithm for the new $Q$-round problem is directly carrying the \textbf{DMFW} algorithm.

Similarly, \textbf{DMFW} algorithm needs to inquire stochastic gradients at $K$ different points for each local average function $\bar{f}_{q,i}$ and then sequentially communicates these gradients' information with the neighboring nodes. 
Noticeably, in $q$-th block of each node $i\in[N]$, there exist $K$ different unbiased gradient oracles $\{\widetilde{\nabla}f_{(q-1)K+1,i},\dots,\widetilde{\nabla}f_{qK,i}\}$. Next, we present a method to generate stochastic gradients of $\bar{f}_{q,i}$ at $K$ different points throughout these $K$ gradient oracles $\{\widetilde{\nabla}f_{(q-1)K+1,i},\dots,\widetilde{\nabla}f_{qK,i}\}$. To be precise, let $\{t_{i}^{(1)}(q),\dots,t_{i}^{(K)}(q)\}$ be a random permutation of the indices $\{(q-1)K+1,\dots,qK\}$ of node $i$. Then, for each $t_{i}^{(k)}(q)$ where $k\in[K]$, we have $\E(f_{t_{i}^{(k)}(q),i}(\x)|\x)=\bar{f}_{q,i}(\x)$ and $\E(\widetilde{\nabla}f_{t_{i}^{(k)}(q),i}(\x)|\x)=\nabla\bar{f}_{q,i}(\x)$. 
As a result, with only one gradient evaluation per function $f_{t_{i}^{(k)}(q),i}$ where $k\in[K]$, we can obtain $K$ unbiased stochastic gradients of the virtual objective function $\bar{f}_{q,i}$.
In this way, for each node $i\in[N]$, we only need to inquire one gradient estimate of the local function $f_{t,i}$ and share this message with the neighbors, which successfully reduce the required number of both per-round communication complexity and per-function gradient evaluations from $K$ to $1$. Merging this blocking technique into the \textbf{DMFW}, we get the one-shot decentralized Meta-Frank-Wolfe~(\textbf{Mono-DMFW}) in Algorithm~\ref{alg:1}.

In Algorithm~\ref{alg:1}, each node $i\in[N]$ at round $q$ keeps track of two local variables $\x_{i}^{(k)}(q)$ and $\dbf_{i}^{(k)}(q)$ which are iteratively updated using the information of the neighboring nodes. After receiving the local update direction $\vbf_{i}^{(k)}(q)$, node $i\in[N]$ updates their local variable $\x_{i}^{(k-1)}(q)$ by averaging their local and neighboring decisions and ascends in the direction $\vbf_{i}^{(k)}(q)$ with stepsize $1/K$~(see line 9), i.e., 
\begin{equation*}
    \x_{i}^{(k)}(q)=\sum_{j\in\mathcal{N}_{i}\cup\{i\}}a_{ij}\x_{j}^{(k-1)}(q)+\frac{1}{K}\vbf_{i}^{(k)}(q),
\end{equation*} where $a_{ij}$ is the weight that node $i$ assigns to node $j$. Similarly, we update local gradient approximation vector $\dbf_{i}^{(k)}(q)$ according to the following rule~(see line 19):
\begin{equation*}
    \dbf_{i}^{(k)}(q)=(1-\gamma)\sum_{j\in\mathcal{N}_{i}\cup\{i\}}a_{ij}\dbf_{j}^{(k-1)}(q)+\gamma\g_{i}^{(k)}(q),
\end{equation*} 
where we view $\g_{i}^{(k)}(q)$ as an approximation to the gradient of $\bar{f}_{q,i}$.

We will show that \textbf{Mono-DMFW} achieves a $(1-1/e)$-regret bound of $O(T ^{4/5})$. Before that, we first state the following assumption.
\begin{assumption}\label{ass5}
For any linear maximization oracle $\mathcal{E}_{i}^{(k)}$, the regret at horizon $t$ is at most $M_{0}\sqrt{t}$, where $M_{0}$ is a parameter.
\end{assumption}
\begin{theorem}[Proof in Appendix~\ref{appendix:one-shot}]\label{thm1}
Under Assumption~\ref{ass1}-\ref{ass5} and $\|\nabla f_{t,i}(\x)\|\le G$, if we set $\eta_{t}=\frac{2}{(t+3)^{2/3}}$ when $1\le t\le\frac{K}{2}+1$, and when $\frac{K}{2}+2\le t\le K$, $\eta_{t}=\frac{1.5}{(K-t+2)^{2/3}}$ in Algorithm~\ref{alg:1}, we have, for each node $j\in[N]$,
\begin{equation*}
    \begin{aligned}
    \E[\mathcal{R}_{\alpha}(T,j)]
    \le \ & NG\mathrm{diam}(\mathcal{K})\frac{Q}{\gamma}+C_{1}Q+C_{2}\frac{Q\log(K+1)}{\gamma}\\
    \ & +C_{3}QK^{2/3}+C_{4}\frac{KQ\gamma}{1-(1-\gamma)\beta}+C_{5}K\sqrt{Q},
    \end{aligned}
\end{equation*} 
where $\alpha=1-1/e$,  $C_{1}=\frac{LNr^{2}(\mathcal{K})}{2}+\frac{Nr(\mathcal{K})(L\mathrm{diam}(\mathcal{K})+GN^{1/2})}{1-\beta}$, $C_{2}=N\mathrm{diam}(\mathcal{K})(2G+Lr(\mathcal{K})$, $C_{3}=2N\mathrm{diam}(\mathcal{K})\sqrt{C_{6}}$, $C_{4}=\mathrm{diam}(\mathcal{K})N\sqrt{2(\sigma^{2}+G^{2})}$, $C_{5}=M_{0}N$, and $C_{6}=\max\{5^{2/3}G^{2},4(G^{2}+\sigma^{2})+32(2G+Lr(\mathcal{K}))^{2},2.25(G^{2}+\sigma^{2})+7(2G+Lr(\mathcal{K}))^{2}/3\}$,
\end{theorem}
\begin{remark}
According to Theorem~\ref{thm1}, if we set $K=T^{3/5}$, $Q=T^{2/5}$, and $\gamma=T^{-1/5}$, \textbf{Mono-DMFW} attains the $(1-1/e)$-regret of $O(T^{4/5})$.
As far as we know, this is the first projection-free decentralized online algorithm to achieve $(1-1/e)$ approximation ratio with communication complexity $1$ for each node at each round, which greatly reduce the $T^{3/2}$ communication complexity of \textbf{DMFW}~\citep{zhu2021projection}.
\end{remark}

%% file: arxiv_v1/One_shot_gradient_ascent.tex
In sharp contrast with \textbf{DMFW} algorithm~\citep{zhu2021projection}, Algorithm~\ref{alg:1} trades the convergence rate for the lower communication overheads. 
In this section, we will present a faster one-shot decentralized online algorithm for continuous DR-submodular maximization problem, i.e., the decentralized online boosting gradient ascent~(\textbf{DOBGA}). 

We begin by reviewing the online projected gradient algorithms. In the \emph{centralized} setting, \citet{zinkevich2003online} first proposed the online gradient descent algorithm and derived the corresponding regret bounds for (strongly) convex objective functions. Then, \citet{chen2018online} proved that the classical online projected gradient method achieves a suboptimal $(1/2)$-approximation guarantee for online continuous DR-submodular maximization problem. Based on a novel auxiliary function, \citet{zhang2022boosting} proceeded to present a variant of the online gradient ascent algorithm to boost the approximation ratio from $1/2$ to $1-1/e$. As for the \emph{decentralized} scenarios, \citet{yan2012distributed} first exhibited a decentralized online projected gradient method, which achieves a square-root regret for convex objectives.

As we know, the decentralized online projected gradient method is naturally one-shot, meaning that every node only needs to inquire one gradient evaluation and communicate with neighbors once at each round. 
However, we cannot directly apply the decentralized online projected gradient method~\citep{yan2012distributed} for the continuous DR-submodular maximization problem, since it will suffer the same suboptimal $(1/2)$-approximation guarantee as its centralized counterpart~\citep{chen2018online}.
To achieve the optimal $(1-1/e)$ approximation ratio, we will design a new decentralized online gradient ascent algorithm based on the boosting auxiliary function.

\subsection{Boosting Auxiliary Function}\label{sec:auxiliary}
In this subsection, we first review some concepts and results about the novel boosting auxiliary function in \cite{zhang2022boosting}. For each monotone continuous DR-submodular function $f:\mathcal{X}\rightarrow\R_{+}$ with $f(\mathbf{0})=0$, we define its boosting auxiliary function as:
\begin{equation}\label{equ:non-oblivious}
    F(\x)=\int_{0}^{1}\frac{e^{z-1}}{z}f(z*\x)dz.
\end{equation} 
Here, for any fixed $\x\in\mathcal{X}$, the boosting function $F$ allocates different weights $\frac{e^{z-1}}{z}$ to the function values $f(z*\x)$. Next, we demonstrate some important properties of $F$.
\begin{lemma}[\citet{zhang2022boosting}]\label{lemma:bianfen}
If the monotone continuous DR-submodular $f$ is differentiable, for any $\x,\y\in\mathcal{X}$, we have $\langle\y-\x,\nabla F(\x)\rangle\ge(1-1/e)f(\y)-f(\x)$.
\end{lemma}

\begin{remark} 
First, we review the definition of the stationary point on domain $\mathcal{K}$. A point $\x\in\mathcal{K}$ is called a stationary point for function $g$ over the domain $\mathcal{K}$ iff $\max_{\y\in\mathcal{K}}\langle\nabla g(\x),\y-\x\rangle\le 0$.
As a result, this lemma demonstrates that any stationary point of $F$ on the domain $\mathcal{K}$ can attain $(1-1/e)$-approximation of the global maximum of $f$ over $\mathcal{K}$. However, according to \citet{hassani2017gradient}, $\langle\y-\x,\nabla f(\x)\rangle\ge\frac{1}{2}f(\y)-f(\x)$ which implies that the stationary point of $f$ itself only provides a $(1/2)$-approximation guarantee.
\end{remark}

\begin{remark}
Moreover, the projected gradient ascent method with small step size usually
converges to a stationary point, which sheds light on the possibility to use auxiliary function $F$ to improve
the performance of the classical gradient ascent algorithm for monotone DR-submodular
maximization. Similarly, in this paper, we hope to use the auxiliary function to boost the approximation ratio of the decentralized online gradient ascent algorithm.
\end{remark}

From Equation~\eqref{equ:non-oblivious}, we know that $\nabla F(\x)=\int_{0}^{1}e^{z-1}\nabla f(z*\x)\mathrm{d}z$. Generally speaking, it is impossible to directly compute the $\nabla F(\x)$ by this equation. Therefore, we proceeded by showing how to estimate $\nabla F(\x)$ using an unbiased stochastic oracle  $\widetilde{\nabla}f(\x)$. We first generate a sample $z$ from a random variable $\mathbf{Z}\in[0,1]$ where Pr$(\mathbf{Z}\le z)=\frac{ e^{z-1}-1/e}{1-1/e}$. Next, we inquire the stochastic gradient at point $z*\x$ of the original function $f$, i.e., $\widetilde{\nabla}f(z*\x)$. Then, we use the $(1-1/e)\widetilde{\nabla}f(z*\x)$ to estimate the $\nabla F(\x)$. In Appendix, we will show the estimate $(1-1/e)\widetilde{\nabla}f(z*\x)$ is unbiased under Assumption~\ref{ass4}.

\begin{algorithm}[t]
	\caption{Decentralized Online Boosting Gradient Ascent~(\textbf{DOBGA})}\label{alg:2}
	\begin{algorithmic}[1]
	\STATE{\bf Input:} Time horizon $T$, the number of nodes $N$, weight matrix $\A=[a_{ij}]\in\R^{N\times N}$, step size $\eta_{t},\forall t\in[T]$
	\STATE {\bf Output:} $\{\x_{i}(t): i\in[N],t\in[T]\}$.
	\STATE Initialize a point $\x_{i}(1)\in\K$ for any $i\in[N]$.
	\FOR{$t=1,\dots,T$}
	\FOR{$i=1,\dots,N$}
		\STATE Node $i$ plays $\x_{i}(t)$ to get reward $f_{t,i}(\x_{i}(t))$ and observes the stochastic gradient information of $f_{t,i}$.
		\STATE Generate a sample $z_{i}(t)$ of the random variable $\mathbf{Z}\in[0,1]$ where $\mathrm{Pr}(\mathbf{Z}\le z)=\frac{e^{z-1}-1/e}{1-1/e}$.
		\STATE Query the stochastic gradient $\widetilde{\nabla}f_{t,i}(z_{i}(t)*\x_{i}(t))$.
		\STATE $\y_{i}(t+1)=\sum_{j\in\mathcal{N}_{i}\cup\{i\}}a_{ij}\x_{j}(t)+\eta_{t}(1-1/e)\widetilde{\nabla}f_{t,i}(z_{i}(t)*\x_{i}(t))$.
		\STATE $\x_{i}(t+1)=\arg\min_{\z\in\K}\left\|\z-\y_{i}(t+1)\right\|$.
		\ENDFOR
		\ENDFOR
	\end{algorithmic}
\end{algorithm}

\subsection{Decentralized Online Boosting Gradient Ascent}
In this subsection, we first assume that, for each local function $f_{t,i}$, $f_{t,i}(\mathbf{0})=0$. Note that the shift of each local function will not affect our theoretical analysis and algorithm. Also, we set the boosting auxiliary function of each local $f_{t,i}$ as $F_{t,i}(\x)=\int_{0}^{1}\frac{e^{z-1}}{z}f_{t,i}(z*\x)dz$. Next, we introduce our proposed decentralized online boosting gradient ascent~(\textbf{DOBGA}) as shown in Algorithm~\ref{alg:2}.

In Algorithm~\ref{alg:2}, each node $i\in[N]$ first generates a sample $z_{i}(t)$ and inquires the stochastic gradient at point $z_{i}(t)*\x_{i}(t)$ to compute the stochastic gradient of boosting auxiliary function $F_{t,i}$ (See line $7$-$8$). Then, \textbf{DOBGA} algorithm aggregates the actions of neighboring nodes and its own decision, namely, $\sum_{j\in\mathcal{N}_{i}\cup\{i\}}a_{ij}\x_{j}(t)$, and pushes the aggregated information along the stochastic boosting gradient with stepsize $\eta_{t}$ (See line $9$). Finally, the node $i$ updates its action via the projection operation (See line $10$). Obviously, in Algorithm~\ref{alg:2}, node $i$ only inquires one stochastic gradient evaluation and communicates with neighboring nodes once at each round. We proceed to show the regret bound for Algorithm~\ref{alg:2}.
\begin{theorem}[Proof in Appendix~\ref{appendix:boosting}]\label{thm:2}
Under Assumption~\ref{ass1}-\ref{ass4}, if $\eta_{t}\ge\eta_{t+1}$ for each $1\le t<[T]$ in Algorithm~\ref{alg:2} and $\|\widetilde{\nabla}f_{t,i}(\x)\|\le G_{1}$, then we have, for each node $j\in[N]$,
\begin{equation*}
    \begin{aligned}
     \E[\mathcal{R}_{\alpha}(T,j)]\le\frac{N\mathrm{diam}(\K)}{2\eta_{T}}+C_{7}\sum_{t=1}^{T}\eta_{t},
    \end{aligned}
\end{equation*} 
where $\alpha=1-1/e$, and $C_{7}=2NG_{1}^{2}+\frac{4NG_{1}^{2}}{1-\beta}+\frac{2G^{2}_{1}(N+N^{3/2})}{(1-\beta)\beta}$.
\end{theorem}
\begin{remark}If we set stepsize $\eta_{t}=\frac{1}{\sqrt{t}}$, \textbf{DOBGA} achieves the $(1-1/e)$-regret of $O(\sqrt{T})$. To the best of our knowledge, this is the first decentralized online algorithm to obtain $O(\sqrt{T})$ regret bound against a $(1-1/e)$-approximation with $O(1)$ gradient inquiry \blue{and communication complexity} for each node $i\in[N]$ per step.
\end{remark}

%% file: arxiv_v1/Experiment.tex
In this section, we will empirically compare the following algorithms:

\noindent\textbf{Decentralized Meta Frank-Wolfe~(\textbf{DMFW})}: Algorithm $1$ in \citet{zhu2021projection} with $K=T^{3/2}$, $\eta_{k}=2/K^{2/3}$ and $\gamma_{k}=1/K^{1/2}$.\\
\noindent\textbf{One-Shot Decentralized Meta-Frank-Wolfe~(\textbf{Mono-DMFW})}: Our proposed Algorithm~\ref{alg:1} with $\eta_{k}=\frac{2}{(k+3)^{2/3}}$ for any $1\le k\le\frac{K}{2}+1$ and $\eta_{k}=\frac{1.5}{(K-k+2)^{2/3}}$ for any $\frac{K}{2}+2\le k\le K$, $\gamma=1/T^{1/5}$, $K=T^{3/5}$ as well as $Q=T^{2/5}$.\\
\noindent\textbf{Decentralized Online Boosting Gradient Ascent~(\textbf{DOBGA})}: Our proposed Algorithm~\ref{alg:2} with $\eta_{t}=1/\sqrt{t}$. Furthermore, we use the average of $5$ independent stochastic gradients to estimate each $\nabla F_{t,i}$ at every iteration.


\begin{figure*}[t]
	\centering
	\subfigure[Complete graph \label{graph1}]{\includegraphics[width=0.28\linewidth]{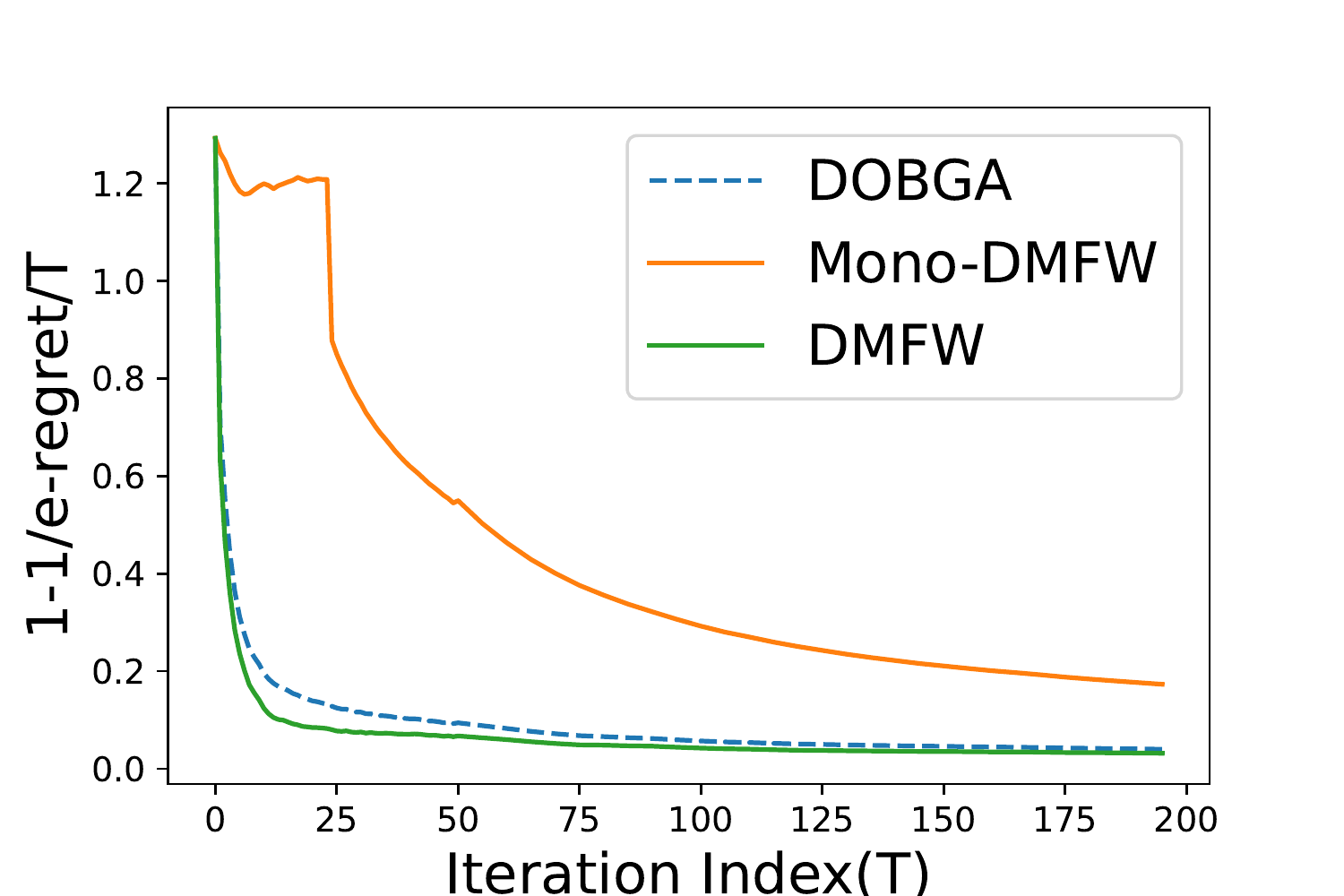}}
	\subfigure[Cycle graph\label{graph2}]{\includegraphics[width=0.28\linewidth]{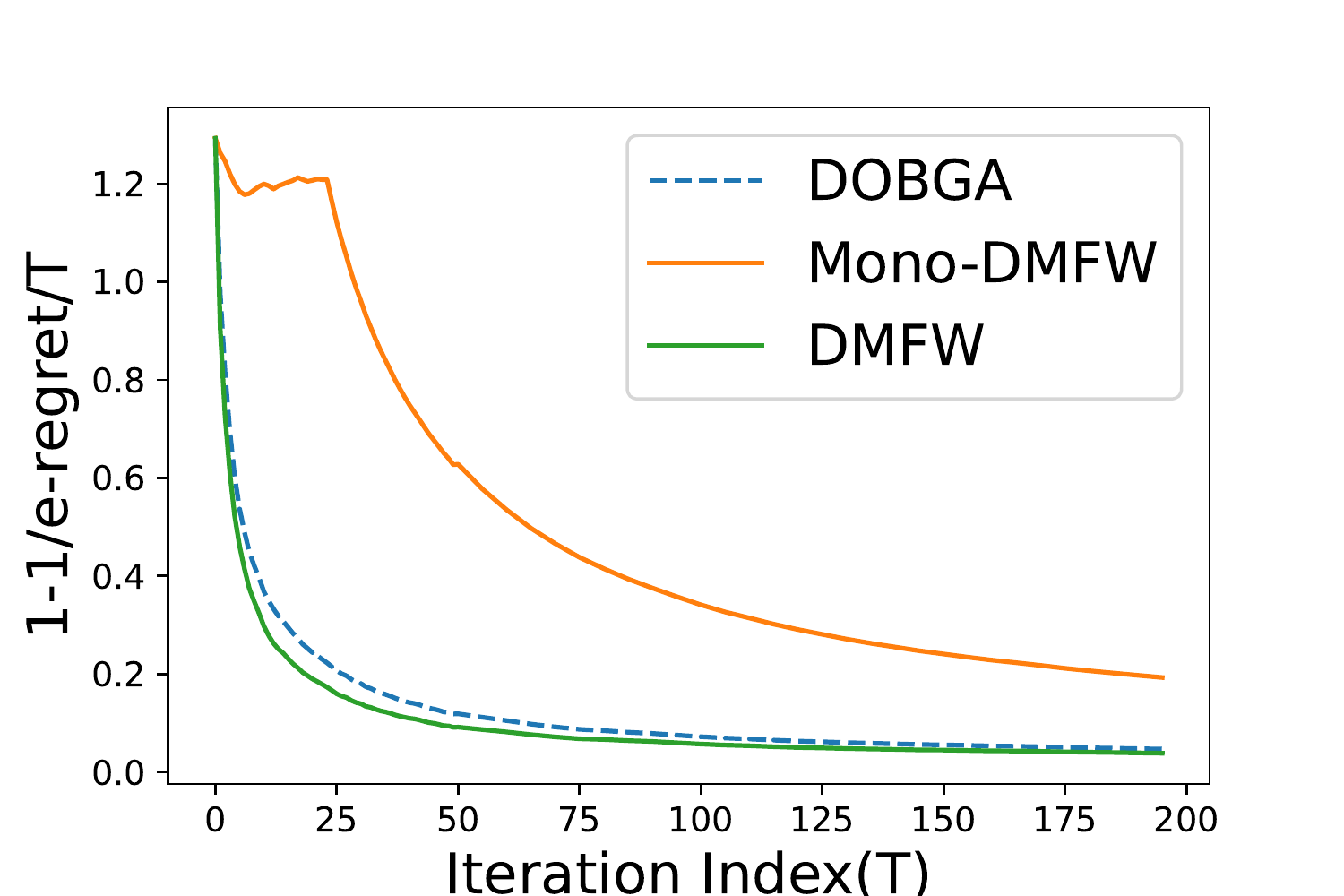}}
	\subfigure[ER graph\label{graph3}]{\includegraphics[width=0.28\linewidth]{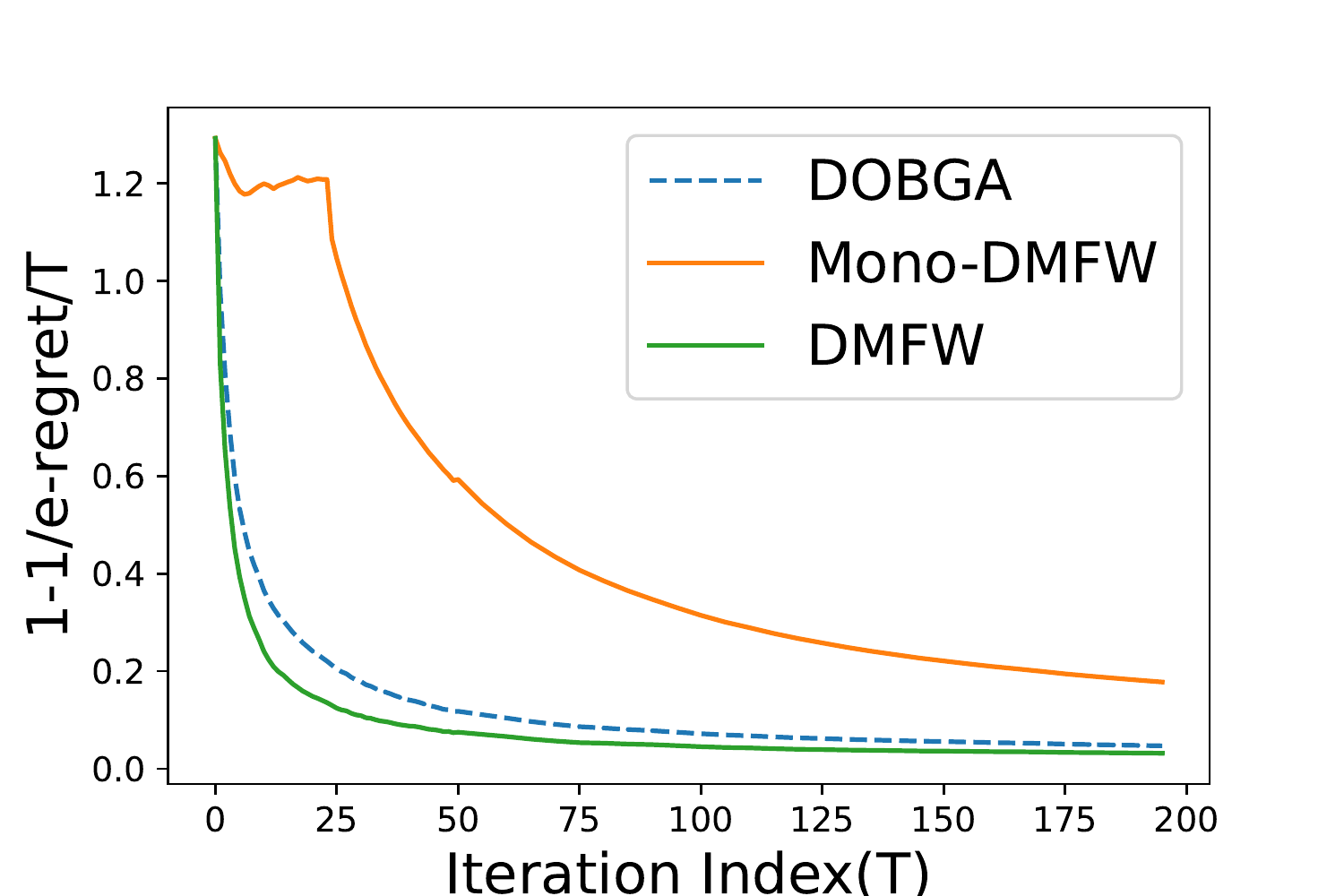}}
	\subfigure[\textbf{Mono-DMFW}\label{graph4}]{\includegraphics[width=0.24\linewidth]{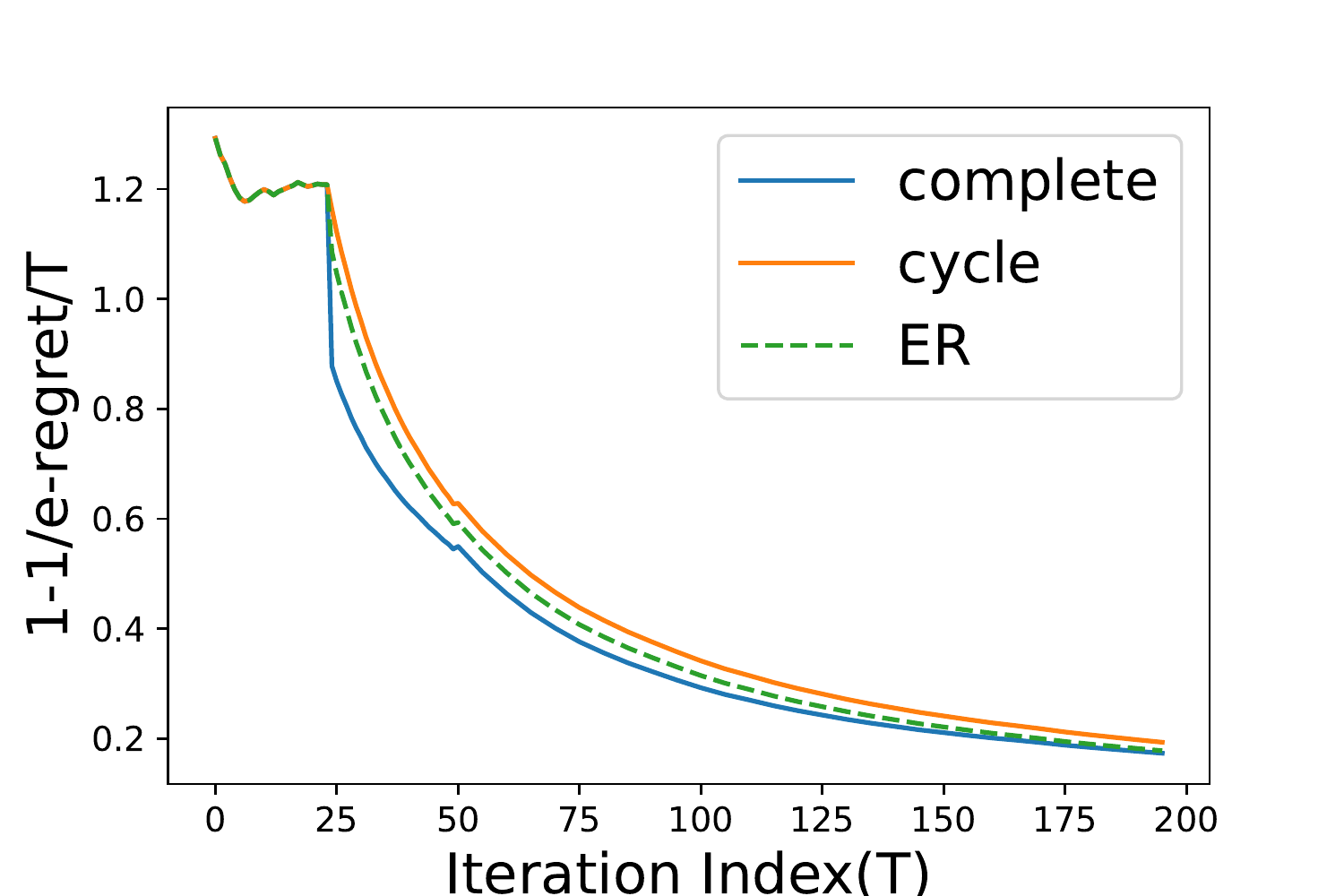}}
	\subfigure[\textbf{DOBGA}\label{graph5}]{\includegraphics[width=0.24\linewidth]{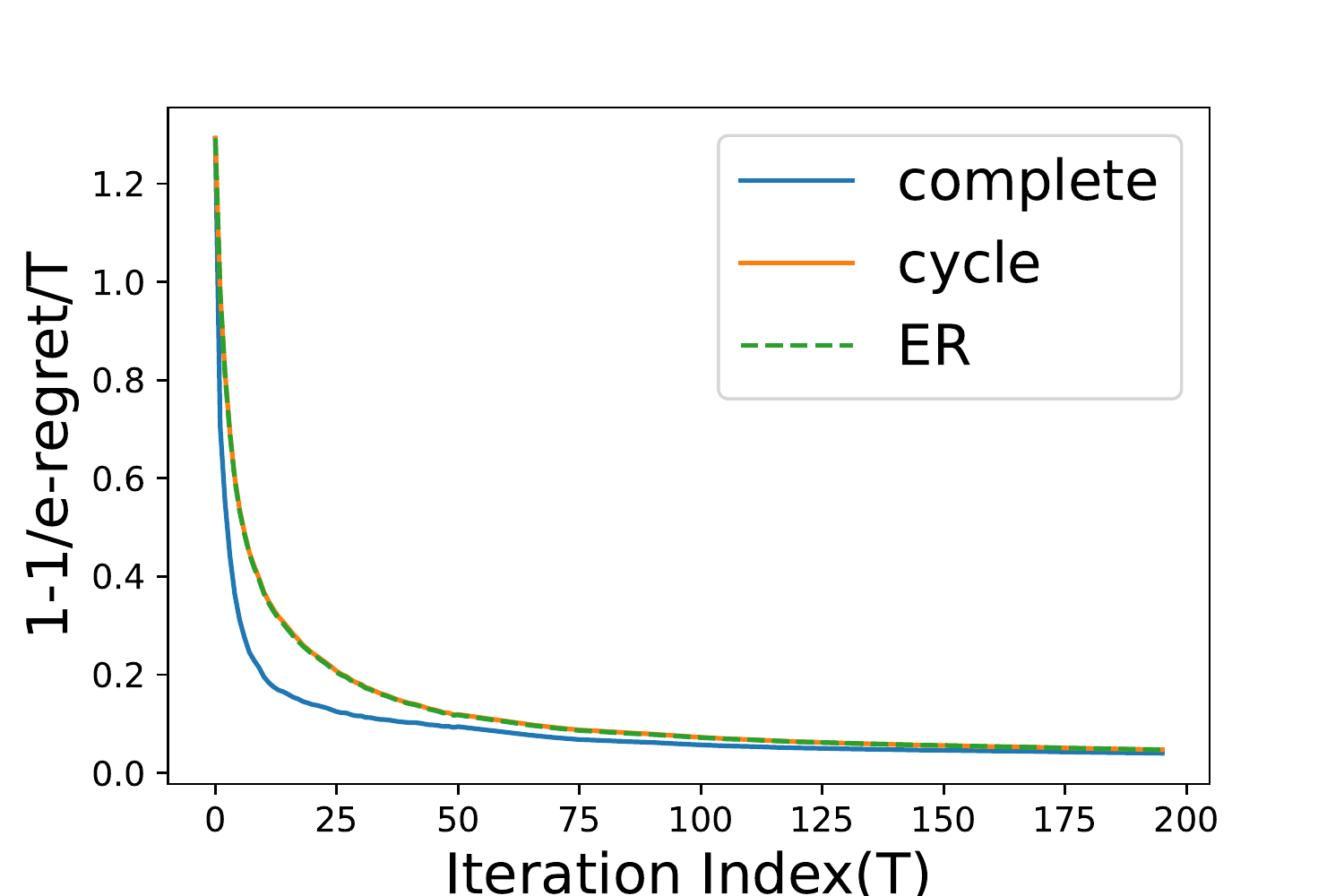}}
	\subfigure[\textbf{Mono-DMFW}\label{graph6}]{\includegraphics[width=0.24\linewidth]{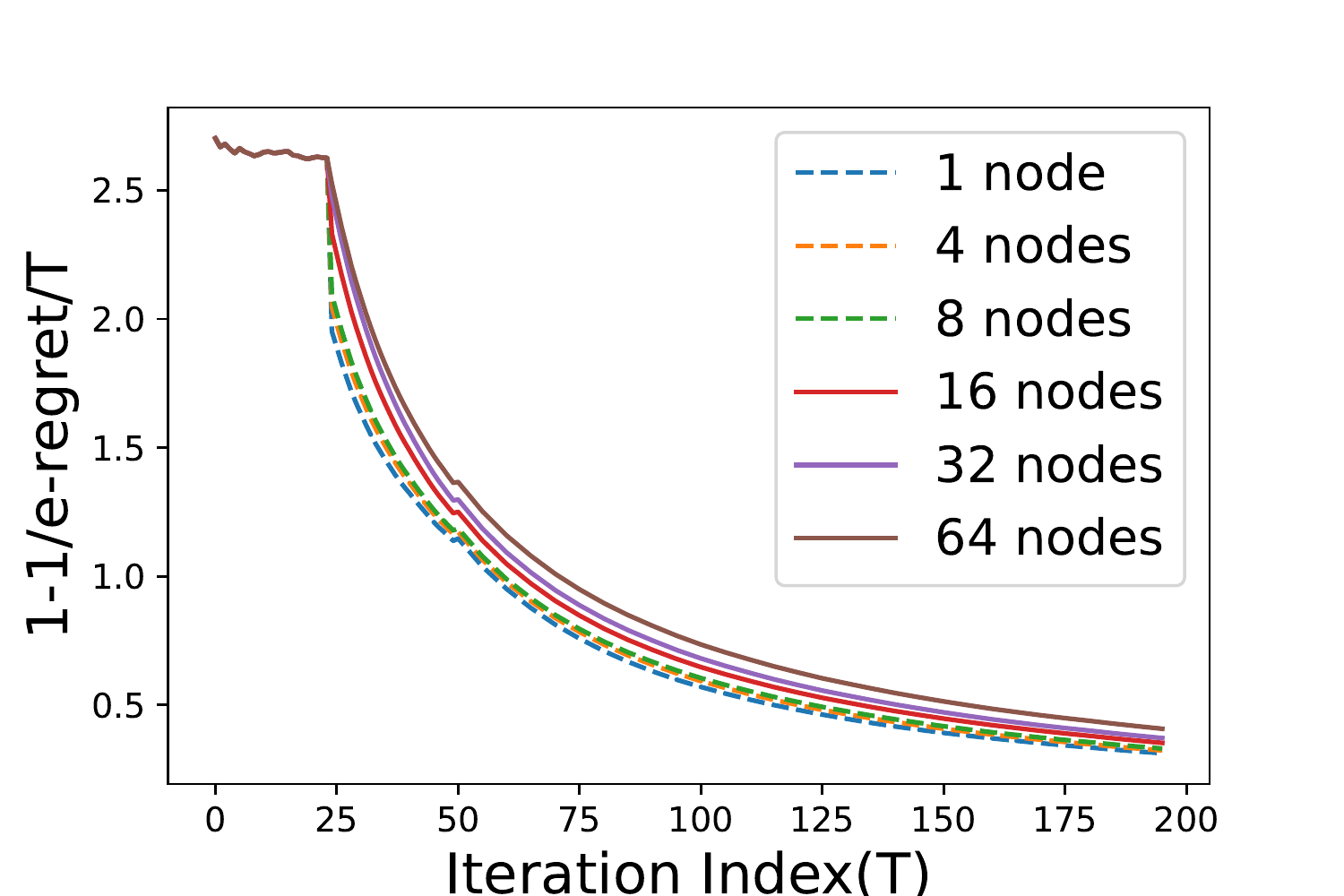}}
	\subfigure[\textbf{DOBGA}\label{graph7}]{\includegraphics[width=0.24\linewidth]{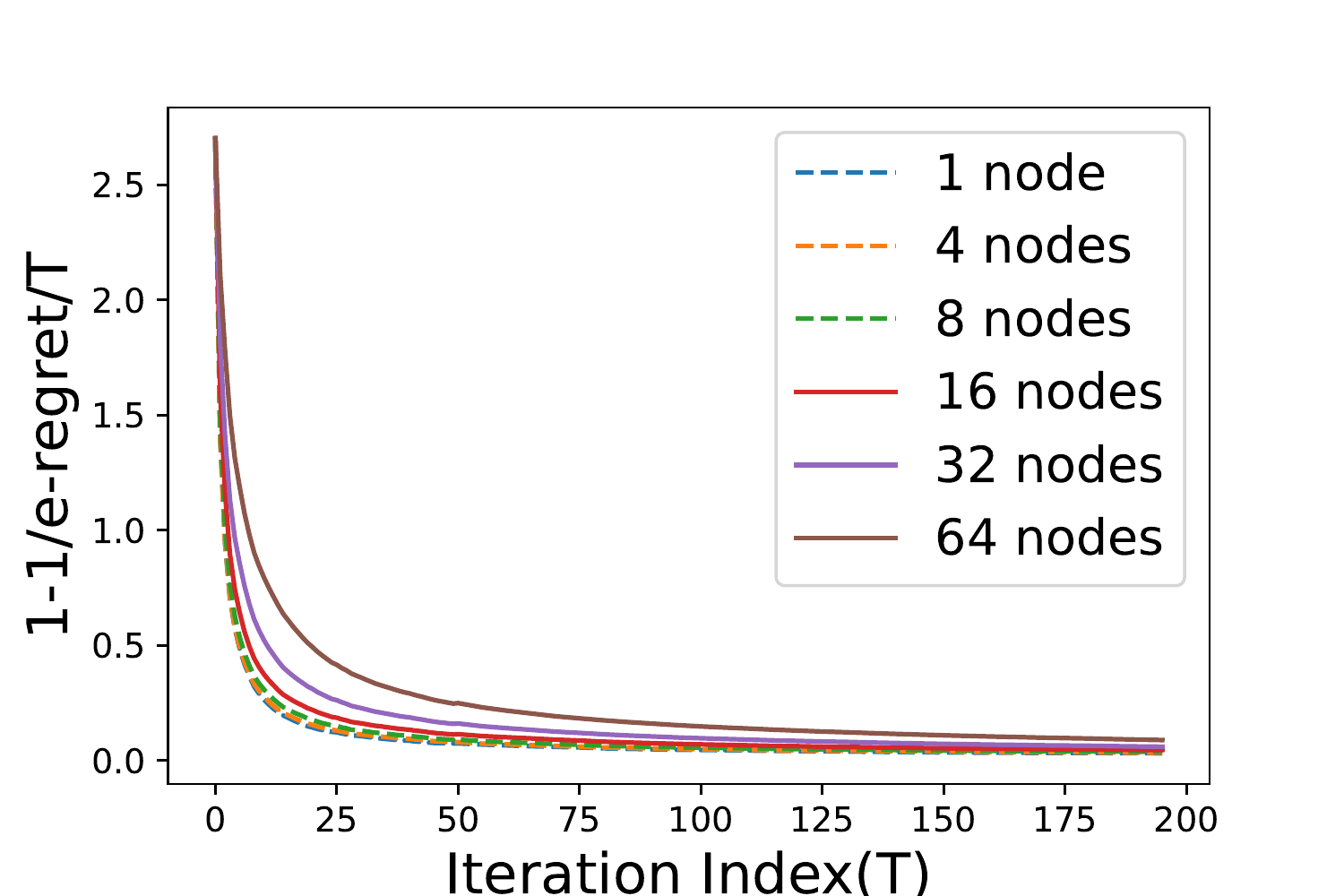}}
	\caption{Figure~\ref{graph1}-\ref{graph3} compares the performance of \textbf{DMFW} and our proposed algorithms(\textbf{Mono-DMFW} and \textbf{DOBGA}) in complete graph, cycle graph and Erdos-Renyi random graph. In Figure~\ref{graph4}, we focus on \textbf{Mono-DMFW} algorithm and compare its trend of the ratio between $(1-1/e)$-regret and iteration index in three graphs. Figure~\ref{graph5} compares the performance of \textbf{DOBGA} in three graphs. Figure~\ref{graph6}-\ref{graph7} show the performance of \textbf{Mono-DMFW} and \textbf{DOBGA} on complete graphs with different number of nodes.}
\end{figure*}

We evaluate these algorithms on the MovieLens dataset~\citep{harper2015movielens} for the movie recommendation task. This dataset consists of $10$ million $5$-star ratings by $72,000$ users for $10,000$ movies. All Ratings are made with half-star increment. Note that all experiments are performed in Python 3.6.5 using CVX optimization tool~\citep{grant2014cvx} on a Windows 10 machine with 32GB RAM and Intel(R)255
Core(TM) i7-9700K CPU.

To begin, we split the first $T*b$ users into disjoint and equally-sized sets $\mathcal{U}_{1},\dots,\mathcal{U}_{T}$, i.e., $|\mathcal{U}_{t}|=b$ for any $t\in[T]$. At the $t$-th round, we assume the data of the users $\mathcal{U}_{t}$ is distributed equally over the network of nodes, which means that, at each round $t$, each node $i\in[N]$ only can get access to the data of a partition of $\mathcal{U}_{t}$. Furthermore, we use $\mathcal{U}_{t,i}$ to denote the users stored in node $i$ at round $t$. In this experiment, our goal is to find a set of $k$ movies that are satisfactory to all considered users.

Precisely, let $r_{u,m}$ denote the rating of user $u$ for movie $m$. For each user $u$, we consider a well-motivated facility location objective function $g_{u}(S)=\max_{m\in S}r_{u,m}$ where $S$ is a subset of the movies. Such a function shows how much the user $u$ is satisfied with the movies in the subset $S$. Due to that each $\mathcal{U}_{t,i}$ contains multiple users, the facility location function associated with node $i$ at round $t$ is $g_{t,i}(S)=\sum_{u\in\mathcal{U}_{t,i}}g_{u}(S)$. Following \citet{zhu2021projection}, at each round $t$, we also consider setting the objective function of node $i$ as the multilinear extension of the facility location function $g_{t,i}$, i.e., $f_{t,i}(\x)=\sum_{S\subseteq[n]}g_{t,i}(S)\prod_{j\in S}(\x)_{j}\prod_{j\in[n]\setminus S}(1-(\x)_{j})$, where we encode the movies as the set $\{1,2,\dots,n\}$. According to \citet{zhu2021projection}, $f_{t,i}$ is a monotone continuous DR-submodular function. 

In the first experiment, we plan to recommend a set of ten movies with high objective value from the first $200$ movies at each round, so $n=200$ and the constraint set is $\K=\{\x\in\R^{n}|\mathbf{0}\le\x\le\mathbf{1},\mathbf{1}^{T}\x\le 10\}$. We also consider a network of $30$ nodes and set $b=60$ and $T=200$. The stochastic gradient is obtained by imposing the Gaussian noise, i.e., $\widetilde{\nabla}f_{t,i}(\x)=\nabla f_{t,i}(\x)+0.1*\mathcal{N}(0,\mathbf{I})$ for any $t\in[T]$, $i\in[N]$ and $\x\in[0,1]^{n}$, where $\mathcal{N}(0,\mathbf{I})$ is standard multivariate normal distribution. As for the network topology, we consider three different choices: A complete graph, a cycle graph, and an Erdos-Renyi random graph (with an average degree of $3$). The matrix $\A$ is chosen in the following manner. If the edge $(i,j)$ is an edge of the graph, let $a_{ij}=1/(1+\max(d_{i},d_{j}))$ where $d_{i}$ and $d_{j}$ are the degree of node $i$ and $j$, respectively. If $(i,j)$ is not an edge of the graph and $i\neq j$, then $a_{ij}=0$. Finally, we set $a_{ii}=1-\sum_{j\in\mathcal{N}_{i}}a_{ij}$. According to \citet{mokhtari2018decentralized}, $\A$ is satisfied with Assumption~\ref{ass1}.


We present the trend of the ratio between $(1-1/e)$-regret and timestamp in the Figure~\ref{graph1}-\ref{graph5} and report the running time in Table~\ref{tab:2}. From Figure~\ref{graph1}-\ref{graph3}, we observe that the \textbf{Mono-DMFW} and \textbf{DMFW} algorithm exhibit the lowest and fastest convergence rate in all three graphs, respectively. This phenomenon is in accord with our theoretical analysis~(See Table~\ref{tab:compare}). When the iteration index is large, the \textbf{DOBGA} achieves nearly the same $(1-1/e)$-regret with \textbf{DMFW}. Despite the better empirical performance of \textbf{DMFW}, both \textbf{Mono-DMFW} and \textbf{DOBGA} efficiently speed up the computation.
As shown in Table~\ref{tab:2}, our proposed \textbf{DOBGA} and \textbf{Mono-DMFW} can be $400$ and $3000$ times faster than \textbf{DMFW}, respectively, in all three graphs.

In Figure~\ref{graph4}-\ref{graph5}, we present how the network topology affects the performance of the \textbf{Mono-DMFW} and \textbf{DOBGA} algorithms. 
We can observe that both algorithms converges faster on the complete graph than the cycle graph and ER graph.

Following \citet{zhu2021projection}, we also explore how the number of nodes influences Algorithm~\ref{alg:1} and Algorithm~\ref{alg:2}. In this experiment, we consider the complete graph with different number of nodes and set $b=128$. As shown in Figure~\ref{graph6}-\ref{graph7}, when the number of nodes increases, the $(1-1/e)$-regret of both \textbf{Mono-DMFW} and \textbf{DOBGA} decrease more slowly, which confirms our theoretical results.

\begin{table}[t]
	\centering
	\caption{Running time~(in seconds)}
    \begin{tabular}{c|c|c|c}
    \toprule[1.5pt]
     Method& Complete&Cycle&ER\\
    \hline 
    \textbf{DMFW} & $824438.1$ & $815161.5$ & $818915.5$\\
    \hline
    \hline
    \textbf{Mono-DMFW} & $275.7$ & $260.7$ & $264.0$\\
    \hline 
    \textbf{DOBGA} & $1898.0$ & $1865.7$ & $1885.7$\\
    \midrule[1.5pt]
    \end{tabular}
	\label{tab:2}
\end{table}

%% file: arxiv_v1/Appendix1.tex
In this section, we present the proof of Theorem~\ref{thm1}. To begin, we define some auxiliary vectors.
\begin{equation*}
    \begin{aligned}
    &\bar{\x}^{(k)}(q)=\frac{\sum_{i=1}^{N}\x_{i}^{(k)}(q)}{N},\\
    &\bar{\dbf}^{(k)}(q)=\frac{\sum_{i=1}^{N}\dbf_{i}^{(k)}(q)}{N}.\\
    \end{aligned}
\end{equation*}
First, we establish an upper bound between $\bar{\x}^{(k+1)}(q)$ and $\bar{\x}^{(k)}(q)$.
\begin{lemma}\label{thm1.1}
Under Assumption~\ref{ass1}, for each $0\le k\le K-1$ and $q\in[Q]$, we have
$\|\bar{\x}^{(k+1)}(q)-\bar{\x}^{(k)}(q)\|\le\frac{r(\K)}{K}$.
\end{lemma}
\begin{proof}
\begin{equation*}
    \begin{aligned}
        \bar{\x}^{(k+1)}(q)&=\frac{\sum_{i=1}^{N}\x_{i}^{(k+1)}(q)}{N}\\
        &=\frac{\sum_{i=1}^{N}(\sum_{j\in\mathcal{N}_{i}\cup\{i\}}a_{ij}\x_{j}^{(k)}(q)+\frac{\vbf_{i}^{(k+1)}(q)}{K})}{N}\\
        &=\frac{\sum_{j=1}^{N}\sum_{i\in\mathcal{N}_{j}\cup\{j\}}a_{ij}\x_{j}^{(k)}(q)}{N}+\frac{\sum_{i=1}^{N}\vbf_{i}^{(k+1)}(q)}{KN}\\
        &=\frac{\sum_{i=1}^{N}\x_{i}^{(k)}(q)}{N}+\frac{\sum_{i=1}^{N}\vbf_{i}^{(k+1)}(q)}{KN}\\
        &=\bar{\x}^{(k)}(q)+\frac{\sum_{i=1}^{N}\vbf_{i}^{(k+1)}(q)}{KN}.
    \end{aligned}
\end{equation*}
Due to $\vbf_{i}^{(k+1)}(q)\in\mathcal{K}$, 
\begin{equation*}
   \|\bar{\x}^{(k+1)}(q)-\bar{\x}^{(k)}(q)\|\le\|\frac{\sum_{i=1}^{N}\vbf_{i}^{(k+1)}(q)}{KN}\|\le\frac{r(\mathcal{K})}{K}.
   \end{equation*}
\end{proof}
Next, we derive the upper bound of the distance between $\x_{i}^{(k)}(q)$ and its average vector $\bar{\x}^{(k)}(q)$. At first, we recall the definition of the second largest magnitude of the eigenvalues of matrix $\A$, i.e., 
\begin{definition}\label{def1}
Consider the eigenvalues of $\A$ which can be sorted in a non-increasing order as $1=\lambda_{1}(\A)\ge\lambda_{2}(\A)\dots\ge\lambda_{n}(\A)>-1$. Define $\beta$ as the second largest magnitude of the eigenvalues of $\A$, i.e., $\beta=\max(|\lambda_{2}(\A)|,|\lambda_{n}(\A)|)$.
\end{definition}

As a result, we have
\begin{lemma}\label{thm1.2}
Under Assumption~\ref{ass1}, for any $k\in[K]$ and $q\in[Q]$, we could conclude that
\begin{equation*}
    \sqrt{\sum_{i=1}^{N}\|\x_{i}^{(k)}(q)-\bar{\x}^{(k)}(q)\|^{2}}\le\frac{\sqrt{N}r(\mathcal{K})}{K(1-\beta)}.
\end{equation*}
\end{lemma}
\begin{proof} 
In order to verify Lemma~\ref{thm1.2},  we define two auxiliary concatenation vectors:
\begin{equation*}
    \begin{aligned}
       \x^{(k)}(q)&=[\x_{1}^{(k)}(q);\dots;\x_{N}^{(k)}(q)]\in\R^{Nn},\\ 
       \vbf^{(k)}(q)&=[\vbf_{1}^{(k)}(q);\dots;\vbf_{N}^{(k)}(q)]\in\R^{Nn}.\\ 
    \end{aligned}.
   \end{equation*}
According to Algorithm~\ref{alg:1}, \begin{equation*}
\begin{aligned}
   \x^{(k)}(q)&=(\A\otimes \mathbf{I})\x^{(k-1)}(q)+\vbf^{(k)}(q)/K\\
   &=\sum_{s=1}^{k}(\A\otimes \mathbf{I})^{k-s}\vbf^{(s)}(q)/K,
\end{aligned}
    \end{equation*} where $\mathbf{I}$ is $n$-dimensional identity matrix and the second equality comes from the iteration.
Simultaneously, we could get the average concatenation vector:
\begin{equation*}
\begin{aligned}
   (\frac{\mathbf{1}\mathbf{1}^{T}}{N}\otimes\mathbf{I})\x^{(k)}(q)&=[\bar{\x}^{(k)}(q);\dots;\bar{\x}^{(k)}(q)]\\
   &=\sum_{s=1}^{k}(\frac{\mathbf{1}\mathbf{1}^{T}}{N}\otimes \mathbf{I})\vbf^{(s)}(q)/K,
\end{aligned} 
    \end{equation*}where $\mathbf{1}$ is $N$-dimensional vector, every element of which is $1$; and the second equality follows from $\mathbf{1}^{T}A=\mathbf{1}^{T}$.
Finally, 
\begin{equation*}
    \begin{aligned}
       \sqrt{\sum_{i=1}^{N}\|\x_{i}^{(k)}(q)-\bar{\x}^{(k)}(q)\|^{2}}&=\|\x^{(k)}(q)-(\frac{\mathbf{1}\mathbf{1}^{T}}{N}\otimes \mathbf{I})\x^{(k)}(q)\|\\
       &=\|\sum_{s=1}^{k}(\frac{(\mathbf{1}\mathbf{1}^{T}}{N}-\A^{k-s})\otimes\mathbf{I})\vbf^{(s)}(q)/K\|\\
       &\le\sum_{s=1}^{k}\|\frac{\mathbf{1}\mathbf{1}^{T}}{N}-\A^{k-s}\|\|\vbf^{(s)}(q)/K\|\\
       &\le\sum_{s=1}^{k}\beta^{k-s}\frac{\sqrt{N}r(\mathcal{K})}{K}\\
       &\le\frac{\sqrt{N}r(\mathcal{K})}{K(1-\beta)},
       \end{aligned}
\end{equation*} where the second inequality follows from $\|\frac{\mathbf{1}\mathbf{1}^{T}}{N}-\A^{k-s}\|\le\beta^{k-s}$~\citep{mokhtari2018decentralized}.
\end{proof}

\begin{lemma}\label{thm1.3}
Under Assumption~\ref{ass1}-\ref{ass4}, for any $k\in[K]$, $q\in[Q]$ and $i\in[N]$, we have
$\E(\|\g_{i}^{(k)}(q)\|^{2})\le 2(\sigma^{2}+G^{2})$.
\end{lemma}
\begin{proof}
Fixed the $i$ and $q$, we prove this lemma by induction. When $k=1$, 
\begin{equation*}
    \begin{aligned}
    & \E(\g_{i}^{(1)}(q))\\
    = & \E(\eta_{1}\widetilde{\nabla}f_{t^{(k)}_{i}(q),i}(\x_{i}^{(k)}(q)))\\
    \le & \E(\|\widetilde{\nabla}f_{t^{(k)}_{i}(q),i}(\x_{i}^{(k)}(q))\|^{2})\\
    \le & \E(2(\|\widetilde{\nabla}f_{t^{(k)}_{i}(q),i}(\x_{i}^{(k)}(q))-\nabla f_{t^{(k)}_{i}(q),i}(\x_{i}^{(k)}(q))\|^{2}+\|\nabla f_{t^{(k)}_{i}(q),i}(\x_{i}^{(k)}(q))\|^{2}))\\
    \le & 2(\sigma^{2}+G^{2}),
    \end{aligned}
\end{equation*} where the first inequality comes from $\eta_{1}\le 1$, and the final inequality follows from the $f_{t^{(k)}_{i}(q),i}$ is $G$-Lipschitz and Assumption~\ref{ass4}.

Then, if $k=m$, $\E(\|\g_{i}^{(m)}(q)\|^{2})\le 2(\sigma^{2}+G^{2})$, we have
\begin{equation*}
    \begin{aligned}
       \E(\|\g_{i}^{(m+1)}(q)\|^{2})=& \E(\|(1-\eta_{m})\g_{i}^{(m)}(q)+\eta_{m}\widetilde{\nabla}f_{t^{(m)}_{i}(q),i}(\x_{i}^{(m)}(q))\|^{2})\\
       = & (1-\eta_{m})^{2}\E(\|\g_{i}^{(m)}(q)\|^{2})+\eta_{m}^{2}\E(\|\widetilde{\nabla}f_{t^{(m)}_{i}(q),i}(\x_{i}^{(m)}(q)\|^{2}))\\
       & +2(1-\eta_{m})\eta_{m}\E(\langle \g_{i}^{(m)}(q),\widetilde{\nabla}f_{t^{(m)}_{i}(q),i}(\x_{i}^{(m)}(q))\rangle)\\
       \le & (1-\eta_{m})^{2}\E(\|\g_{i}^{(m)}(q)\|^{2})+\eta_{m}^{2}\E(\|\widetilde{\nabla}f_{t^{(m)}_{i}(q),i}(\x_{i}^{(m)}(q)\|^{2}))\\&+2(1-\eta_{m})\eta_{m}\E(\frac{\|\g_{i}^{(m)}(q)\|^{2}+\|\widetilde{\nabla}f_{t^{(m)}_{i}(q),i}(\x_{i}^{(m)}(q)\|^{2}}{2})\\
       = & 2(\sigma^{2}+G^{2})[(1-\eta_{m})^{2}+2\eta_{m}(1-\eta_{m})+\eta_{m}^{2}]\\
       \le & 2(\sigma^{2}+G^{2}),
    \end{aligned}
\end{equation*} 
where the first inequality comes from the Cauchy–Schwarz inequality; the second inequality follows from $\E(\|\g_{i}^{(m)}(q)\|^{2})\le 2(\sigma^{2}+G^{2})$ and $\E(\|\widetilde{\nabla}f_{t^{(m)}_{i}(q),i}(\x_{i}^{(m)}(q)\|^{2})\le2(\sigma^{2}+G^{2})$.
\end{proof}
\begin{lemma}\label{thm1.4}
Under Assumption~\ref{ass1}-\ref{ass4}, for any $k\in[K]$ and $q\in[Q]$, we could conclude that
\begin{equation*}
    \E(\sqrt{\sum_{i=1}^{N}\|\dbf_{i}^{(k)}(q)-\bar{\dbf}^{(k)}(q)\|^{2}})\le\frac{\gamma\sqrt{2N(\sigma^{2}+G^{2})}}{1-(1-\gamma)\beta}.
\end{equation*}
\end{lemma}
\begin{proof}
We first define some variables:
\begin{equation*}
\begin{aligned}
    \dbf^{(k)}(q)&=[\dbf_{1}^{(k)}(q);\dots;\dbf_{N}^{(k)}(q)]\in\R^{Nn}\\
    \g^{(k)}(q)&=[\g_{1}^{(k)}(q);\dots;\g_{N}^{(k)}(q)]\in\R^{Nn}.
\end{aligned}
   \end{equation*}
According to Algorithm~\ref{alg:1},  we have
\begin{equation*}
\begin{aligned}
   \dbf^{(k)}(q)&=(1-\gamma)(\A\otimes \mathbf{I})\dbf^{(k-1)}(q)+\gamma \g^{(k)}(q)\\
   &=\gamma\sum_{s=1}^{k}(1-\gamma)^{k-s}(\A\otimes\mathbf{I})^{k-s}\g^{(s)}(q),
\end{aligned}
    \end{equation*} where $\mathbf{I}$ is a $n$-dimensional identity matrix. Simultaneously, we have 
\begin{equation*}
\begin{aligned}
   (\frac{\mathbf{1}\mathbf{1}^{T}}{N}\otimes\mathbf{I})\dbf^{(k)}(q)&=[\bar{\dbf}^{(k)}(q);\dots;\bar{\dbf}^{(k)}(q)]\\
   &=\gamma\sum_{s=1}^{k}(1-\gamma)^{k-s}(\frac{\mathbf{1}\mathbf{1}^{T}}{N}\otimes\mathbf{I})g^{(s)}(q),
\end{aligned}
    \end{equation*} where $\mathbf{1}$ is $N$-dimensional vector, whose element is $1$, and the second equality follows from $\mathbf{1}^{T}A=\mathbf{1}^{T}$.
    Finally, 
\begin{equation}
    \begin{aligned}
       \E(\sqrt{\sum_{i=1}^{N}\|\dbf_{i}^{(k)}(q)-\bar{\dbf}^{(k)}(q)\|^{2}})&=\E(\|\dbf^{(k)}(q)-(\frac{\mathbf{1}\mathbf{1}^{T}}{N}\otimes I)\dbf^{(k)}(q)\|)\\
       &=\E(\|\sum_{s=1}^{k}(\gamma(1-\gamma)^{k-s}\frac{(\mathbf{1}\mathbf{1}^{T}}{N}-\A^{k-s})\otimes \mathbf{I})\g^{(s)}(q)\|)\\
       &\le\sum_{s=1}^{k}\gamma(1-\gamma)^{k-s}\|\frac{\mathbf{1}\mathbf{1}^{T}}{N}-\A^{k-s}\|\E(\|\g^{(s)}(q)\|)\\
       &\le\sum_{s=1}^{k}\gamma(1-\gamma)^{k-s}\beta^{k-s}\sqrt{2N(\sigma^{2}+G^{2})}\\
       &\le\frac{\gamma\sqrt{2N(\sigma^{2}+G^{2})}}{1-(1-\gamma)\beta},
       \end{aligned} 
\end{equation} where the second inequality follows from $\|\frac{\mathbf{1}\mathbf{1}^{T}}{N}-\A^{k-s}\|\le\beta^{k-s}$~\citep{mokhtari2018decentralized} and Lemma~\ref{thm1.3}.
\end{proof}
Before presenting the regret bound, we introduce some notations. We first define the average function as:
\begin{equation*}
 \bar{f}_{q,i}=\frac{\sum_{k=1}^{K}f_{(q-1)K+k,i}}{K}.
\end{equation*}
Let
\begin{equation*}
\bar{f}_{q,k,i}=\frac{\sum_{i=k+1}^{K}f_{t_{i}^{(m)}(q)},i}{K-k}.
\end{equation*} denotes the average function of the remaining $(K-k)$ functions for node $i$ in $q$-th block, where $0\le k\le K-1$.
With these notions, we could obtain that
\begin{lemma}\label{thm1.5}
Under Assumption~\ref{ass3}, for any $q\in[Q]$ and $0\le k\le K-1$, we have that
\begin{equation*}
    \begin{aligned}
    & \frac{1}{N}\sum_{i=1}^{N}\bar{f}_{q,k,i}(\x^{*})-\frac{1}{N}\sum_{i=1}^{N}\bar{f}_{q,k,i}(\bar{\x}^{(k+1)}(q))\\
    \le & (1-\frac{1}{K})(\frac{1}{N}\sum_{i=1}^{N}\bar{f}_{q,k,i}(\x^{*})-\frac{1}{N}\sum_{i=1}^{N}\bar{f}_{q,k,i}(\bar{\x}^{(k)}(q)))+\frac{1}{K}\langle\bar{\dbf}^{(k)}(q),\x^{*}-\bar{\vbf}^{(k)}(q)\rangle\\\
    & +\frac{\mathrm{diam}(\mathcal{K})}{KN}\|\sum_{i=1}^{N}\nabla\bar{f}_{q,k,i}(\bar{\x}^{(k)}(q))-N\bar{d}^{(k)}(q)\|+\frac{L}{2K^{2}}\|\bar{\vbf}^{(k)}(q)\|^{2},
    \end{aligned}
\end{equation*} where $\x^{*}=\arg\max_{\x\in\mathcal{K}}\sum_{t=1}^{T}\sum_{i=1}^{N}f_{t,i}(\x)$.
\end{lemma}
\begin{proof}
$f_{t,i}$ is $L$-smooth, monotone, and DR-submodular, so is $\bar{f}_{q,k,i}$. As a result, we have
\begin{equation*}
    \begin{aligned}
       &\frac{1}{N}\sum_{i=1}^{N}\bar{f}_{q,k,i}(\bar{\x}^{(k+1)}(q))-\frac{1}{N}\sum_{i=1}^{N}\bar{f}_{q,k,i}(\bar{\x}^{(k)}(q))\\
       \ge & \langle\frac{1}{N}\sum_{i=1}^{N}\nabla\bar{f}_{q,k,i}(\bar{\x}^{(k)}(q)), \bar{\x}^{(k+1)}(q)-\bar{\x}^{(k)}(q)\rangle-\frac{L}{2}\|\bar{\x}^{(k+1)}(q)-\bar{\x}^{(k)}(q)\|^{2}\\
       \ge & \frac{1}{KN^{2}}\langle\sum_{i=1}^{N}\nabla\bar{f}_{q,k,i}(\bar{\x}^{(k)}(q)), \sum_{i=1}^{N}\vbf_{i}^{(k+1)}(q)\rangle-\frac{L}{2K^{2}N^{2}}\|\sum_{i=1}^{N}\vbf_{i}^{(k+1)}(q)\|^{2}\\
       = & \frac{1}{KN}\langle\sum_{i=1}^{N}\nabla\bar{f}_{q,k,i}(\bar{\x}^{(k)}(q)),\x^{*}\rangle+\frac{1}{K}\langle\bar{\dbf}^{(k)}(q),\bar{\vbf}^{(k)}(q)-\x^{*}\rangle\\
       & +\frac{1}{KN}\langle\sum_{i=1}^{N}\nabla\bar{f}_{q,k,i}(\bar{\x}^{(k)}(q))-N\bar{\dbf}^{(k)}(q),\bar{\vbf}^{(k)}(q)-\x^{*}\rangle-\frac{L}{2K^{2}}\|\bar{\vbf}^{(k)}(q)\|^{2}\\
       \ge & \frac{1}{KN}\sum_{i=1}^{N}\bar{f}_{q,k,i}(\x^{*})-\frac{1}{KN}\sum_{i=1}^{N}\bar{f}_{q,k,i}(\bar{\x}^{(k)}(q)) +\frac{1}{K}\langle\bar{\dbf}^{(k)}(q),\bar{\vbf}^{(k)}(q)-\x^{*}\rangle\\
       & +\frac{1}{KN}\langle\sum_{i=1}^{N}\nabla\bar{f}_{q,k,i}(\bar{\x}^{(k)}(q))-N\bar{\dbf}^{(k)}(q),\bar{\vbf}^{(k)}(q)-\x^{*}\rangle-\frac{L}{2K^{2}}\|\bar{\vbf}^{(k)}(q)\|^{2},
    \end{aligned}
\end{equation*} where $\langle\nabla\bar{f}_{q,k,i}(\bar{\x}^{(k)}(q)),\x^{*}\rangle\ge\bar{f}_{q,k,i}(\x^{*})-\bar{f}_{q,k,i}(\bar{\x}^{(k)}(q))$~\citep{mokhtari2020stochastic}.
Therefore, 
\begin{equation*}
    \begin{aligned}
        & \frac{1}{N}\sum_{i=1}^{N}\bar{f}_{q,k,i}(\x^{*})-\frac{1}{N}\sum_{i=1}^{N}\bar{f}_{q,k,i}(\bar{\x}^{(k+1)}(q))\\
        \le & (1-\frac{1}{K})(\frac{1}{N}\sum_{i=1}^{N}\bar{f}_{q,k,i}(\x^{*})-\frac{1}{N}\sum_{i=1}^{N}\bar{f}_{q,k,i}(\bar{\x}^{(k)}(q)))+\frac{1}{K}\langle\bar{\dbf}^{(k)}(q),\x^{*}-\bar{\vbf}^{(k)}(q)\rangle\\
        & +\frac{\mathrm{diam}(\mathcal{K})}{KN}\|\sum_{i=1}^{N}\nabla\bar{f}_{q,k,i}(\bar{\x}^{(k)}(q))-N\bar{d}^{(k)}(q)\|+\frac{L}{2K^{2}}\|\bar{\vbf}^{(k)}(q)\|^{2}.
    \end{aligned}
\end{equation*}
\end{proof}
\begin{lemma}\label{thm1:6}
Under Assumption~\ref{ass3}, we have that
\begin{equation*}
    \begin{aligned}
    & \sum_{t=1}^{T}\sum_{i=1}^{N}(1-1/e)f_{t,i}(\x^{*})-\sum_{t=1}^{T}\sum_{i=1}^{N}\E(f_{t,i}(\bar{\x}(t)))\\
    \le & N\sum_{q=1}^{Q}\sum_{s=0}^{K-1}(1-\frac{1}{K})^{K-1-s}\E(\langle\bar{\dbf}^{(s)}(q),\x^{*}-\bar{\vbf}^{(s)}(q)\rangle)+\frac{LNQr^{2}(\mathcal{K})}{2}\\
    & +\mathrm{diam}(\mathcal{K})\sum_{q=1}^{Q}\sum_{s=0}^{K-1}(1-\frac{1}{K})^{K-1-s}\E(\|\sum_{i=1}^{N}\nabla\bar{f}_{q,s,i}(\bar{\x}^{(s)}(q))-N\bar{\dbf}^{(s)}(q)\|),
    \end{aligned}
\end{equation*} where $\bar{x}(t)=\sum_{i=1}^{N}\x_{i}(t)/N$ for any $t\in[T]$ and $\x_{i}(t)$ is played action for $f_{t,i}$~(See line $14$ in Algorithm~\ref{alg:1}).
\end{lemma}
\begin{proof}  
\begin{equation*}
\begin{aligned}
    & \E(\frac{1}{N}\sum_{i=1}^{N}\bar{f}_{q,i}(\x^{*}))-\E(\frac{1}{N}\sum_{i=1}^{N}\bar{f}_{q,i}(\bar{\x}^{(K)}(q)))\\
    = & \E(\frac{1}{N}\sum_{i=1}^{N}\bar{f}_{q,K-1,i}(\x^{*}))-\E(\frac{1}{N}\sum_{i=1}^{N}\bar{f}_{q,K-1,i}(\bar{\x}^{(K)}(q)))\\
    \le & (1-\frac{1}{K})(\frac{1}{N}\sum_{i=1}^{N}\E(\bar{f}_{q,K-1,i}(\x^{*}))-\frac{1}{N}\sum_{i=1}^{N}\E(\bar{f}_{q,K-1,i}(\bar{\x}^{(K-1)}(q)))+\frac{1}{K}\E(\langle\bar{\dbf}^{(K-1)}(q),\x^{*}-\bar{\vbf}^{(K-1)}(q)\rangle)\\
    & +\frac{\mathrm{diam}(\mathcal{K})}{KN}\E(\|\sum_{i=1}^{N}\nabla\bar{f}_{q,K-1,i}(\bar{\x}^{(K-1)}(q))-N\bar{\dbf}^{(K-1)}(q)\|)+\frac{L}{2K^{2}}\E(\|\bar{\vbf}^{(K-1)}(q)\|^{2})\\
    = & (1-\frac{1}{K})(\frac{1}{N}\sum_{i=1}^{N}\E(\bar{f}_{q,K-2,i}(\x^{*}))-\frac{1}{N}\sum_{i=1}^{N}\E(\bar{f}_{q,K-2,i}(\bar{\x}^{(K-1)}(q)))+\frac{1}{K}\E(\langle\bar{\dbf}^{(K-1)}(q),\x^{*}-\bar{\vbf}^{(K-1)}(q)\rangle)\\
    & +\frac{\mathrm{diam}(\mathcal{K})}{KN}\E(\|\sum_{i=1}^{N}\nabla\bar{f}_{q,K-1,i}(\bar{\x}^{(K-1)}(q))-N\bar{\dbf}^{(K-1)}(q)\|)+\frac{L}{2K^{2}}\E(\|\bar{\vbf}^{(K-1)}(q)\|^{2})\\
    \le & \dots\\
    \le & (1-\frac{1}{K})^{K}\frac{1}{N}\sum_{i=1}^{N}\E(\bar{f}_{q,0,i}(\x^{*}))+\frac{1}{K}\sum_{s=0}^{K-1}(1-\frac{1}{K})^{K-1-s}\E(\langle\bar{\dbf}^{(s)}(q),\x^{*}-\bar{\vbf}^{(s)}(q)\rangle)\\
    & +\frac{\mathrm{diam}(\mathcal{K})}{KN}\sum_{s=0}^{K-1}(1-\frac{1}{K})^{K-1-s}\E(\|\sum_{i=1}^{N}\nabla\bar{f}_{q,s,i}(\bar{\x}^{(s)}(q))-N\bar{\dbf}^{(s)}(q)\|)+\frac{L}{2K^{2}}\sum_{s=0}^{K-1}\E(\|\bar{\vbf}^{(s)}(q)\|^{2})\\
    \le & \frac{1}{e}\frac{1}{N}\sum_{i=1}^{N}\E(\bar{f}_{q,i}(\x^{*}))+\frac{1}{K}\sum_{s=0}^{K-1}(1-\frac{1}{K})^{K-1-s}\E(\langle\bar{\dbf}^{(s)}(q),\x^{*}-\bar{\vbf}^{(s)}(q)\rangle)\\
    & +\frac{\mathrm{diam}(\mathcal{K})}{KN}\sum_{s=0}^{K-1}(1-\frac{1}{K})^{K-1-s}\E(\|\sum_{i=1}^{N}\nabla\bar{f}_{q,s,i}(\bar{\x}^{(s)}(q))-N\bar{\dbf}^{(s)}(q)\|)+\frac{L}{2K^{2}}\sum_{s=0}^{K-1}\E(\|\bar{\vbf}^{(s)}(q)\|^{2}),
    \end{aligned}
\end{equation*} 
where the first equality follows from the $\E(\bar{f}_{q,K-1,i}(\x))|\x)=\bar{f}_{q,i}(\x)$; the first inequality comes from Lemma~\ref{thm1.5}; the second equality follows from the $\E(\bar{f}_{q,K-1,i}(\x))|\x)=\E(\bar{f}_{q,K-2,i}(\x))|\x)$; the final inequality follows from $(1-\frac{1}{K})^{K}\le\frac{1}{e}$ and $\E(\bar{f}_{q,0,i}(\x^{*}))=\E(\bar{f}_{q,i}(\x^{*}))$. 
Finally, we could conclude that 
\begin{equation*}
    \begin{aligned}
    &\sum_{t=1}^{T}\sum_{i=1}^{N}f_{t,i}(\x^{*})-\sum_{t=1}^{T}\sum_{i=1}^{N}\E(f_{t,i}(\bar{\x}(t)))\\
    = & NK(\sum_{q=1}^{Q}\sum_{i=1}^{N}\frac{1}{N}\bar{f}_{q,i}(\x^{*})-\sum_{q=1}^{Q}\sum_{i=1}^{N}\frac{1}{N}\E(\bar{f}_{q,i}(\bar{\x}^{(K)}(q))))\\
    \le & \frac{1}{e}\sum_{t=1}^{T}\sum_{i=1}^{N}f_{t,i}(\x^{*})+N\sum_{q=1}^{Q}\sum_{s=0}^{K-1}(1-\frac{1}{K})^{K-1-s}\E(\langle\bar{\dbf}^{(s)}(q),\x^{*}-\bar{\vbf}^{(s)}(q)\rangle)\\
    & +\mathrm{diam}(\mathcal{K})\sum_{q=1}^{Q}\sum_{s=0}^{K-1}(1-\frac{1}{K})^{K-1-s}\E(\|\sum_{i=1}^{N}\nabla\bar{f}_{q,s,i}(\bar{\x}^{(s)}(q))-N\bar{\dbf}^{(s)}(q)\|)+\frac{LNQr^{2}(\mathcal{K})}{2}.
    \end{aligned}
\end{equation*}
\end{proof}
\begin{lemma}\label{thm1:7}
Under Assumption~\ref{ass1}-\ref{ass5}, we have 
\begin{equation*}
    \sum_{q=1}^{Q}\sum_{s=0}^{K-1}(1-\frac{1}{K})^{K-1-s}\E(\langle\bar{\dbf}^{(s)}(q),\x^{*}-\bar{\vbf}^{(s)}(q)\rangle)\le\frac{\mathrm{diam}(\mathcal{K})KQ\gamma\sqrt{2(\sigma^{2}+G^{2})}}{1-(1-\gamma)\beta}+M_{0}K\sqrt{Q}.
\end{equation*} 
\end{lemma}
\begin{proof}
\begin{equation*}
    \begin{aligned}
    &\sum_{q=1}^{Q}\sum_{s=0}^{K-1}(1-\frac{1}{K})^{K-1-s}\E(\langle\bar{\dbf}^{(s)}(q),\x^{*}-\bar{\vbf}^{(s)}(q)\rangle)\\
    = & \frac{1}{N}\sum_{q=1}^{Q}\sum_{s=0}^{K-1}\sum_{i=1}^{N}(1-\frac{1}{K})^{K-1-s}\E(\langle\bar{\dbf}^{(s)}(q),\x^{*}-\vbf_{i}^{(s)}(q)\rangle)\\
    = & \frac{1}{N}\sum_{q=1}^{Q}\sum_{s=0}^{K-1}\sum_{i=1}^{N}(1-\frac{1}{K})^{K-1-s}\E(\langle\bar{\dbf}^{(s)}(q)-\dbf_{i}^{(s)}(q),\x^{*}-\vbf_{i}^{(s)}(q)\rangle+\langle\dbf_{i}^{(s)}(q),\x^{*}-\vbf_{i}^{(s)}(q)\rangle)\\
    \le & \sum_{q=1}^{Q}\sum_{s=0}^{K-1}(1-\frac{1}{K})^{K-1-s}\E(\frac{\mathrm{diam}(\mathcal{K})}{N}\sum_{i=1}^{N}\|\bar{\dbf}^{(s)}(q)-\dbf_{i}^{(s)}(q)\|+\frac{1}{N}\sum_{i=1}^{N}\langle\dbf_{i}^{(s)}(q),\x^{*}-\vbf_{i}^{(s)}(q)\rangle)\\
    \le & \sum_{q=1}^{Q}\sum_{s=0}^{K-1}(1-\frac{1}{K})^{K-1-s}\frac{\mathrm{diam}(\mathcal{K})}{N}\sqrt{N}\E(\sqrt{\sum_{i=1}^{N}\|\bar{\dbf}^{(s)}(q)-\dbf_{i}^{(s)}(q)\|^{2}})\\
    & +\sum_{i=1}^{N}\sum_{s=0}^{K-1}(1-\frac{1}{K})^{K-1-s}\frac{1}{N}\sum_{q=1}^{Q}\E(\langle\dbf_{i}^{(s)}(q),\x^{*}-\vbf_{i}^{(s)}(q)\rangle)\\
    \le & \sum_{q=1}^{Q}\sum_{s=0}^{K-1}(1-\frac{1}{K})^{K-1-s}\frac{\mathrm{diam}(\mathcal{K})}{N}\frac{N\gamma\sqrt{2(\sigma^{2}+G^{2})}}{1-(1-\gamma)\beta}+\sum_{i=1}^{N}\sum_{s=0}^{K-1}(1-\frac{1}{K})^{K-1-s}\frac{1}{N}M_{0}\sqrt{Q}\\
    \le & \frac{\mathrm{diam}(\mathcal{K})KQ\gamma\sqrt{2(\sigma^{2}+G^{2})}}{1-(1-\gamma)\beta}+M_{0}K\sqrt{Q},
    \end{aligned}
\end{equation*} 
where the third inequality follow from Lemma~\ref{thm1.4} and Assumption~\ref{ass5}.
\end{proof}

Next, we derive the upper bound of $\|\sum_{i=1}^{N}\nabla\bar{f}_{q,s,i}(\bar{\x}^{(s)}(q))-N\bar{\dbf}^{(s)}(q)\|$. Before that, we investigate the gap between $\nabla\bar{f}_{q,k,i}(\bar{\x}_{i}^{(k)}(q))$ and $\nabla\bar{f}_{i,q,k-1}(\bar{\x}_{i}^{(k-1)}(q))$ for any $0\le k\le K-1$.
\begin{lemma}\label{thm1:8}
Under Assumption~\ref{ass3}, we have
\begin{equation*}
    \|\nabla\bar{f}_{q,k,i}(\bar{\x}_{i}^{(k)}(q))-\nabla\bar{f}_{q,k-1,i}(\bar{\x}_{i}^{(k-1)}(q))\|\le\frac{2G+Lr(\mathcal{K})}{K-k+1},
\end{equation*} where $0\le k\le K-1$.
\end{lemma}
\begin{proof}
\begin{equation*}
    \begin{aligned}
       & \nabla\bar{f}_{q,k,i}(\bar{\x}_{i}^{(k)}(q))-\nabla\bar{f}_{q,k-1,i}(\bar{\x}_{i}^{(k-1)}(q))\\
       = & \frac{\sum_{m=k+1}^{K}\nabla f_{t_{i}^{(m)}(q),i}(\bar{\x}_{i}^{(k)}(q))}{K-k}-\frac{\sum_{m=k}^{K}\nabla f_{t_{i}^{(m)}(q),i}(\bar{\x}_{i}^{(k-1)}(q))}{K-k+1}\\
      = & \frac{\sum_{m=k+1}^{K}\nabla f_{t_{i}^{(m)}(q),i}(\bar{\x}_{i}^{(k)}(q))}{(K-k)(K-k+1)}+\frac{\sum_{m=k+1}^{K}\nabla f_{t_{i}^{(m)}(q),i}(\bar{\x}_{i}^{(k)}(q)) -\nabla f_{t_{i}^{(m)}(q),i}(\bar{\x}_{i}^{(k-1)}(q))}{K-k+1}\\
      & -\frac{\nabla f_{t_{i}^{(k)}(q),i}(\bar{\x}_{i}^{(k-1)}(q))}{K-k+1}.
    \end{aligned}
\end{equation*}
Therefore,
\begin{equation*}
    \begin{aligned}
       & \|\nabla\bar{f}_{q,k,i}(\bar{\x}_{i}^{(k)}(q))-\nabla\bar{f}_{q,k-1,i}(\bar{\x}_{i}^{(k-1)}(q))\|\\
       \le & \|\frac{\sum_{m=k+1}^{K}\nabla f_{t_{i}^{(m)}(q),i}(\bar{\x}_{i}^{(k)}(q))-\nabla f_{t_{i}^{(m)}(q),i}(\bar{\x}_{i}^{(k-1)}(q))}{K-k+1}\|\\
       & + \|\frac{\sum_{m=k+1}^{K}\nabla f_{t_{i}^{(m)}(q),i}(\bar{\x}_{i}^{(k)}(q))}{(K-k)(K-k+1)}\|+\|\frac{\nabla f_{t_{i}^{(k)}(q),i}(\bar{\x}_{i}^{(k-1)}(q))}{K-k+1}\|\\
       \le & \frac{(K-k)L\|\bar{\x}_{i}^{(k-1)}(q)-\bar{\x}_{i}^{(k)}(q)\|}{K-k+1}+\frac{(K-k)G}{(K-k)(K-k+1)}+\frac{G}{K-k+1}\\
       \le & \frac{(K-k)Lr(\mathcal{K})}{K(K-k+1)}+\frac{2G}{K-k+1}\\
       \le & \frac{2G+Lr(\mathcal{K})}{K-k+1},
       \end{aligned} 
\end{equation*}
where the  third inequality follows from $\|\bar{\x}_{i}^{(k-1)}(q)-\bar{\x}_{i}^{(k)}(q)\|\le\frac{r(\mathcal{K})}{K}$~(Lemma~\ref{thm1.1}).
\end{proof}

\begin{lemma}\label{thm1:9}
Under Assumption~\ref{ass1}-\ref{ass5}, we could conclude that
\begin{equation}\label{thm1:9.equ}
    \begin{aligned}
    & \|\sum_{i=1}^{N}\nabla\bar{f}_{q,s,i}(\bar{\x}^{(s)}(q))-N\bar{\dbf}^{(s)}(q)\|\\
    \le & (1-\gamma)^{s}NG+\sum_{m=1}^{s}(1-\gamma)^{s-m}\frac{N(2G+Lr(\mathcal{K}))}{K-m+1}\\
    & +\gamma\sum_{m=1}^{s}(1-\gamma)^{s-m}\|\sum_{i=1}^{N}\g_{i}^{(m)}(q)-\sum_{i=1}^{N}\nabla\bar{f}_{q,m-1,i}(\x_{i}^{(m)}(q))\|+\frac{LNr(\mathcal{K})}{K(1-\beta)}.
    \end{aligned}
\end{equation}
\end{lemma}
\begin{proof}
\begin{equation*}
    \begin{aligned}
    & \|\sum_{i=1}^{N}\nabla\bar{f}_{q,s,i}(\bar{\x}^{(s)}(q))-N\bar{\dbf}^{(s)}(q)\|=\|\sum_{i=1}^{N}\nabla\bar{f}_{q,s,i}(\bar{\x}^{(s)}(q))-\sum_{i=1}^{N}\dbf_{i}^{(s)}(q)\|\\
    = & \|\sum_{i=1}^{N}\nabla\bar{f}_{q,s,i}(\bar{\x}^{(s)}(q))-(1-\gamma)\sum_{i=1}^{N}\dbf_{i}^{(s-1)}(q)-\gamma\sum_{i=1}^{N}\g_{i}^{(s)}(q)\|\\
    = & \|(1-\gamma)(\sum_{i=1}^{N}\dbf_{i}^{(s-1)}(q)-\sum_{i=1}^{N}\nabla\bar{f}_{q,s-1,i}(\bar{\x}^{(s-1)}(q)))\\
    & +(\sum_{i=1}^{N}\nabla\bar{f}_{q,s-1,i}(\bar{\x}^{(s-1)}(q))-\sum_{i=1}^{N}\nabla\bar{f}_{q,s,i}(\bar{\x}^{(s)}(q)))\\
    & +\gamma(\sum_{i=1}^{N}\g_{i}^{(s)}(q)
    -\sum_{i=1}^{N}\nabla\bar{f}_{q,s-1,i}(\x_{i}^{(s)}(q)))\\
    & +\gamma(\sum_{i=1}^{N}\nabla\bar{f}_{q,s-1,i}(\x_{i}^{(s)}(q))-\sum_{i=1}^{N}\nabla\bar{f}_{q,s-1,i}(\bar{\x}^{(s)}(q)))\|\\
    \le & (1-\gamma)\|\sum_{i=1}^{N}\dbf_{i}^{(s-1)}(q)-\sum_{i=1}^{N}\nabla\bar{f}_{q,s-1,i}(\bar{\x}^{(s-1)}(q))\|+\frac{N(2G+Lr(\mathcal{K}))}{K-s+1}\\
    & +\gamma\|\sum_{i=1}^{N}\g_{i}^{(s)}(q)
    -\sum_{i=1}^{N}\nabla\bar{f}_{q,s-1,i}(\x_{i}^{(s)}(q))\|+L\gamma\sum_{i=1}^{N}\|\x_{i}^{(s)}(q)-\bar{\x}^{(s)}(q)\|\\
    \le & (1-\gamma)\|\sum_{i=1}^{N}\dbf_{i}^{(s-1)}(q)-\sum_{i=1}^{N}\nabla\bar{f}_{q,s-1,i}(\bar{\x}^{(s-1)}(q))\|+\frac{N(2G+Lr(\mathcal{K}))}{K-s+1}\\
    & +\gamma\|\sum_{i=1}^{N}\g_{i}^{(s)}(q)
    -\sum_{i=1}^{N}\nabla\bar{f}_{q,s-1,i}(\x_{i}^{(s)}(q))\|+L\sqrt{N}\gamma\sqrt{\sum_{i=1}^{N}\|\x_{i}^{(s)}(q)-\bar{\x}^{(s)}(q)\|^{2}}\\
    \le & (1-\gamma)\|\sum_{i=1}^{N}\dbf_{i}^{(s-1)}(q)-\sum_{i=1}^{N}\nabla\bar{f}_{q,s-1,i}(\bar{\x}^{(s-1)}(q))\|+\frac{N(2G+Lr(\mathcal{K}))}{K-s+1}\\
    & +\gamma\|\sum_{i=1}^{N}\g_{i}^{(s)}(q)
    -\sum_{i=1}^{N}\nabla\bar{f}_{q,s-1,i}(\x_{i}^{(s)}(q))\|+\frac{LNr(\mathcal{K})\gamma}{K(1-\beta)},
    \end{aligned}
\end{equation*} where the first inequality follows from Lemma~\ref{thm1:8} and $L$-smoothness of each $\bar{f}_{q,s-1,i}$; the second inequality comes from the Cauchy–Schwarz inequality; the final inequality comes from Lemma~\ref{thm1.2}.

If we set the $\triangle_{s}=\|\sum_{i=1}^{N}\nabla\bar{f}_{q,s,i}(\bar{\x}^{(s)}(q))-N\bar{\dbf}^{(s)}(q)\|$, we have $\triangle_{s}\le\triangle_{s-1}+\gamma\|\sum_{i=1}^{N}\g_{i}^{(s)}(q)-\sum_{i=1}^{N}\nabla\bar{f}_{q,s-1,i}(\x_{i}^{(s)}(q))\|+\frac{LNr(\mathcal{K})\gamma}{K(1-\beta)}+\frac{N(2G+Lr(\mathcal{K}))}{K-s+1}$. By iteration, we have
\begin{equation*}
    \begin{aligned}
    \triangle_{s}\le & (1-\gamma)^{s}NG+\sum_{m=1}^{s}(1-\gamma)^{s-m}\frac{N(2G+Lr(\mathcal{K}))}{K-m+1}\\
    & +\gamma\sum_{m=1}^{s}(1-\gamma)^{s-m}\|\sum_{i=1}^{N}\g_{i}^{(m)}(q)-\sum_{i=1}^{N}\nabla\bar{f}_{q,m-1,i}(\x_{i}^{(m)}(q))\|+\frac{LNr(\mathcal{K})}{K(1-\beta)}.
    \end{aligned}
\end{equation*}
\end{proof}

From Equation~(\ref{thm1:9.equ}), we know that the upper bound of $\|\sum_{i=1}^{N}\nabla\bar{f}_{q,s,i}(\bar{\x}^{(s)}(q))-N\bar{\dbf}^{(s)}(q)\|$ is related with the value of $\|\sum_{i=1}^{N}\g_{i}^{(m)}(q)-\sum_{i=1}^{N}\nabla\bar{f}_{q,m-1,i}(\x_{i}^{(m)}(q))\|$ for any $m\le s$. Next, we derive how each $\g_{i}^{(k)}(q)$ approximate the gradient $\nabla\bar{f}_{q,m-1,i}(\x_{i}^{(k)}(q)$. First, we recall the variance reduction result in \citet{zhang2019online}.
\begin{lemma}[\citet{zhang2019online}]\label{lemma:vr2}
Let $\{a_{k}\}_{k=0}^{K}$ be a sequence of points in $\R^{n}$ such that $\|\a_{k}-\a_{k-1}\|\le\frac{D}{K+2-k}$ for all $1\le k\le K$ with fixed constant $D\ge0$. Let $\{\widetilde{\a}_{k}\}_{k=0}^{K}$ be a sequence of random variables such that $\E(\widetilde{\a}_{k}|\F_{k-1})=\a_{k}$ and $\E(\|\widetilde{\a}_{k}-\a_{k}\|^{2}|\F_{k-1})\le\sigma_{1}^{2}$ for every $k\ge0$, where $\F_{k-1}$ is the $\sigma$-field generated by $\{\widetilde{\a}_{m}\}_{m=0}^{k-1}$ and $\F_{0}=\emptyset$. Let $\{\mathbf{b}_{k}\}_{k=0}^{K}$ be a sequence of random variables where $\mathbf{b}_{0}$ is fixed and subsequent $\mathbf{b}_{k}$ are obtained by $\mathbf{b}_{k}=(1-m_{k})\mathbf{b}_{k-1}+m_{k}\widetilde{a}_{k}$. If we set $m_{k}=\frac{2}{(k+3)^{2/3}}$, when $1\le k\le\frac{K}{2}+1$, and when $\frac{K}{2}+2\le k\le K$, $m_{k}=\frac{1.5}{(K-k+2)^{2/3}}$, we have
\begin{equation*}
    \E(\|\mathbf{b}_{k}-\a_{k}\|^{2})\le\left\{\begin{aligned}
       &\frac{C}{(k+4)^{2/3}}& 1\le k\le\frac{K}{2}+1\\
       &\frac{C}{(K-k+1)^{2/3}}& \frac{K}{2}+2\le k\le K
    \end{aligned}\right.
\end{equation*} where $C=\max\{5^{2/3}\|\a_{0}-\mathbf{b}_{0}\|^{2},4\sigma_{1}^{2}+32D^{2},2.25\sigma_{1}^{2}+7D^{2}/3\}$.
\end{lemma}
\begin{lemma}\label{thm1:10}
Under Assumption~\ref{ass3}-\ref{ass4}, if we set $\eta_{t}=\frac{2}{(t+3)^{2/3}}$ when $1\le t\le\frac{K}{2}+1$, and when $\frac{K}{2}+2\le t\le K$, $\eta_{t}=\frac{1.5}{(K-t+2)^{2/3}}$, we could conclude that
\begin{equation*}
    \begin{aligned}
    \E(\|\g_{i}^{(k)}(q)-\nabla\bar{f}_{q,k-1,i}(\x_{i}^{(k)}(q))\|^{2})\le\left\{\begin{aligned}
       &\frac{C_{1}}{(k+4)^{2/3}}& 1\le k\le\frac{K}{2}+1\\
       &\frac{C_{1}}{(K-k+1)^{2/3}}& \frac{K}{2}+2\le k\le K
    \end{aligned}\right.
    \end{aligned}
\end{equation*}where $C_{1}=\max\{5^{2/3}G^{2},4(G^{2}+\sigma^{2})+32(2G+Lr(\mathcal{K}))^{2},2.25(G^{2}+\sigma^{2})+7(2G+Lr(\mathcal{K}))^{2}/3\}$.
\end{lemma}
\begin{proof}
If, in Lemma~\ref{lemma:vr2}, we set $\widetilde{\a}_{k}=\widetilde{\nabla}f_{t_{i}^{(k)}(q),i}(\x_{i}^{(k)}(q))$ for fixed $q\in[Q]$ and $i\in[N]$, then $\mathbf{b}_{k}$ is equal to $\g_{i}^{(k)}(q)$ in Algorithm~\ref{alg:1}. Furthermore, we set the $\F_{k-1}$~(in Lemma~\ref{lemma:vr2}) is the $\sigma$-field generated by $\{t_{i}^{(1)}(q),\dots,t_{i}^{(k-1)}(q)\}$. Then, we have $\E(\widetilde{\a}_{k}|\F_{k-1})=\a_{k}=\E(\widetilde{\nabla}f_{t_{i}^{(k)}(q),i}(\x_{i}^{(k)}(q))|\F_{k-1})=\nabla\bar{f}_{q,k-1,i}(\x_{i}^{(k)}(q))$. According to Lemma~\ref{thm1:8}, $\|\nabla\bar{f}_{q,k-1,i}(\x_{i}^{(k)}(q))-\nabla\bar{f}_{q,k-2,i}(\x_{i}^{(k-1)}(q))\|\le\frac{2G+Lr(\mathcal{K})}{K-k+2}$. From Lemma $5$ in \citet{zhang2019online}, we also could obtain $\E(\|\widetilde{\nabla}f_{t_{i}^{(k)}(q),i}(\x_{i}^{(k)}(q))-\nabla\bar{f}_{q,k-1,i}(\x_{i}^{(k)}(q))\|^{2}|\F_{k-1})\le G^{2}+\sigma^{2}$. As a result, we have
\begin{equation*}
    \begin{aligned}
    \E(\|\g_{i}^{(k)}(q)-\nabla\bar{f}_{q,k-1,i}(\x_{i}^{(k)}(q))\|^{2})\le\left\{\begin{aligned}
       &\frac{C_{1}}{(k+4)^{2/3}}& 1\le k\le\frac{K}{2}+1\\
       &\frac{C_{1}}{(K-k+1)^{2/3}}& \frac{K}{2}+2\le k\le K
    \end{aligned}\right.
    \end{aligned}
\end{equation*}where $C_{1}=\max\{5^{2/3}G^{2},4(G^{2}+\sigma^{2})+32(2G+Lr(\mathcal{K}))^{2},2.25(G^{2}+\sigma^{2})+7(2G+Lr(\mathcal{K}))^{2}/3\}$.
\end{proof}
After merging the result in Lemma~\ref{thm1:10} into Lemma~\ref{thm1:9}, we get the following lemma.
\begin{lemma}\label{thm1:11}
Under Assumption~\ref{ass1}-\ref{ass5}, we could conclude that
\begin{equation*}
    \begin{aligned}
    &\sum_{q=1}^{Q}\sum_{s=0}^{K-1}(1-\frac{1}{K})^{K-1-s}\E(\|\sum_{i=1}^{N}\nabla\bar{f}_{q,s,i}(\bar{\x}^{(s)}(q))-N\bar{\dbf}^{(s)}(q)\|)\\
    \le&\frac{QNG}{\gamma}+\frac{LNQr(\mathcal{K})}{1-\beta}+\frac{NQ(2G+Lr(\mathcal{K}))\log(K+1)}{\gamma}+2NQ\sqrt{C_{1}}K^{2/3},
    \end{aligned}
\end{equation*} where $C_{1}=\max\{5^{2/3}G^{2},4(G^{2}+\sigma^{2})+32(2G+Lr(\mathcal{K}))^{2},2.25(G^{2}+\sigma^{2})+7(2G+Lr(\mathcal{K}))^{2}/3\}$.
\end{lemma}
\begin{proof}
\begin{equation*}
    \begin{aligned}
    & \sum_{q=1}^{Q}\sum_{s=0}^{K-1}(1-\frac{1}{K})^{K-1-s}\E(\|\sum_{i=1}^{N}\nabla\bar{f}_{q,s,i}(\bar{\x}^{(s)}(q))-N\bar{\dbf}^{(s)}(q)\|)\\
    \le & \sum_{q=1}^{Q}\sum_{s=0}^{K-1}\E(\|\sum_{i=1}^{N}\nabla\bar{f}_{q,s,i}(\bar{\x}^{(s)}(q))-N\bar{\dbf}^{(s)}(q)\|)\\
    \le & \sum_{q=1}^{Q}\sum_{s=0}^{K-1}(1-\gamma)^{s}NG +\sum_{q=1}^{Q}\sum_{s=0}^{K-1}\sum_{m=1}^{s}(1-\gamma)^{s-m}\frac{N(2G+Lr(\mathcal{K}))}{K-m+1}\\
    & +\sum_{q=1}^{Q}\sum_{s=0}^{K-1}\gamma\sum_{m=1}^{s}(1-\gamma)^{s-m}\E(\|\sum_{i=1}^{N}\g_{i}^{(m)}(q)
    -\sum_{i=1}^{N}\nabla\bar{f}_{q,m-1,i}(\x_{i}^{(m)}(q))\|)\\
    & +\sum_{q=1}^{Q}\sum_{s=0}^{K-1}\frac{LNr(\mathcal{K})}{K(1-\beta)}.
    \end{aligned}
\end{equation*}
We first could verify that
\begin{equation*}
  \sum_{q=1}^{Q}\sum_{s=0}^{K-1}(1-\gamma)^{s}NG\le\frac{QNG}{\gamma}.
\end{equation*}
Then, we have
\begin{equation*}
    \begin{aligned}
    \sum_{q=1}^{Q}\sum_{s=0}^{K-1}\frac{LNr(\mathcal{K})}{K(1-\beta)}=\frac{LNQr(\mathcal{K})}{1-\beta}.
    \end{aligned}
\end{equation*}
Next, we prove that
\begin{equation*}
    \begin{aligned}
    &\sum_{q=1}^{Q}\sum_{s=0}^{K-1}\sum_{m=1}^{s}(1-\gamma)^{s-m}\frac{N(2G+Lr(\mathcal{K}))}{K-m+1}\\
    =&\sum_{q=1}^{Q}\sum_{m=0}^{K-1}\sum_{s=m}^{K-1}(1-\gamma)^{s-m}\frac{N(2G+LR)}{K-m+2}\\
    \le&\frac{1}{\gamma}\sum_{q=1}^{Q}\sum_{m=0}^{K-1}\frac{N(2G+Lr(\mathcal{K}))}{K-m+1}\\
    \le&\frac{NQ(2G+Lr(\mathcal{K}))\log(K+1)}{\gamma},
    \end{aligned}
\end{equation*} where the final inequality follows from $\sum_{i=2}^{K+1}1/i\le\log(K+1)$.
Finally, we get 
\begin{equation*}
    \begin{aligned}
    &\sum_{q=1}^{Q}\sum_{s=0}^{K-1}\gamma\sum_{m=1}^{s}(1-\gamma)^{s-m}\E(\|\sum_{i=1}^{N}\g_{i}^{(m)}(q)
    -\sum_{i=1}^{N}\nabla\bar{f}_{q,m-1,i}(\x_{i}^{(m)}(q))\|)\\
    \le&\sum_{q=1}^{Q}\sum_{m=0}^{K-1}\gamma\sum_{s=m}^{K-1}(1-\gamma)^{s-m}\E(\|\sum_{i=1}^{N}\g_{i}^{(m)}(q)
    -\sum_{i=1}^{N}\nabla\bar{f}_{q,m-1,i}(\x_{i}^{(m)}(q))\|)\\
    \le&\sum_{q=1}^{Q}\sum_{m=0}^{K-1}\E(\|\sum_{i=1}^{N}\g_{i}^{(m)}(q)
    -\sum_{i=1}^{N}\nabla\bar{f}_{q,m-1,i}(\x_{i}^{(m)}(q))\|)\\
    \le&\sum_{q=1}^{Q}\sum_{m=0}^{K-1}\sum_{i=1}^{N}\E\|\g_{i}^{(m)}(q)
    -\nabla\bar{f}_{q,m-1,i}(\x_{i}^{(m)}(q))\|)\\
    \le&\sum_{q=1}^{Q}\sum_{m=0}^{K-1}\sqrt{N}\E(\sqrt{\sum_{i=1}^{N}\|\g_{i}^{(m)}(q)
    -\nabla\bar{f}_{q,m-1,i}(\x_{i}^{(m)}(q))\|^{2}})\\
    \le&\sqrt{N}Q\sum_{m=0}^{K-1}\E(\sqrt{\sum_{i=1}^{N}\|\g_{i}^{(m)}(q)
    -\nabla\bar{f}_{q,m-1,i}(\x_{i}^{(m)}(q))\|^{2}})\\
    \le&NQ\sqrt{C_{1}}(\sum_{m=0}^{\frac{K}{2}+1}\frac{1}{(k+4)^{1/3}}+\sum_{m=\frac{K}{2}+2}^{K}\frac{1}{(K-k+1)^{1/3}}\\
    \le&2NQ\sqrt{C_{1}}K^{2/3},
    \end{aligned}
\end{equation*}where the six inequality follows from Lemma~\ref{thm1:10} and $C_{1}=\max\{5^{2/3}G^{2},4(G^{2}+\sigma^{2})+32(2G+Lr(\mathcal{K}))^{2},2.25(G^{2}+\sigma^{2})+7(2G+Lr(\mathcal{K}))^{2}/3\}$; the final inequality comes from the $\sum_{i=1}^{K}1/i^{1/3}\le K^{2/3}$.
\end{proof}

Now, we present the proof of Theorem~\ref{thm1}.
\begin{proof}
Merging Lemma~\ref{thm1:11} and Lemma~\ref{thm1:7} into Lemma~\ref{thm1:6}, we have
\begin{equation*}
    \begin{aligned}
    & \sum_{t=1}^{T}\sum_{i=1}^{N}(1-1/e)f_{t,i}(\x^{*})-\sum_{t=1}^{T}\sum_{i=1}^{N}\E(f_{t,i}(\bar{\x}(t)))\\
    \le & N\sum_{q=1}^{Q}\sum_{s=0}^{K-1}(1-\frac{1}{K})^{K-1-s}\E(\langle\bar{\dbf}^{(s)}(q),\x^{*}-\bar{\vbf}^{(s)}(q)\rangle)+\frac{LNQr^{2}(\mathcal{K})}{2}\\
    & +\mathrm{diam}(\mathcal{K})\sum_{q=1}^{Q}\sum_{s=0}^{K-1}(1-\frac{1}{K})^{K-1-s}\E(\|\sum_{i=1}^{N}\nabla\bar{f}_{q,s,i}(\bar{\x}^{(s)}(q))-N\bar{\dbf}^{(s)}(q)\|)\\
    \le & NG\mathrm{diam}(\mathcal{K})\frac{Q}{\gamma}+(\frac{LNr^{2}(\mathcal{K})}{2}+\frac{LN\mathrm{diam}(\mathcal{K})r(\mathcal{K})}{1-\beta})Q+N\mathrm{diam}(\mathcal{K})(2G+Lr(\mathcal{K}))\frac{Q\log(K+1)}{\gamma}\\
    & + 2N\mathrm{diam}(\mathcal{K})\sqrt{C_{1}}QK^{2/3}+\mathrm{diam}(\mathcal{K})N\sqrt{2(\sigma^{2}+G^{2})}\frac{KQ\gamma}{1-(1-\gamma)\beta}+M_{0}NK\sqrt{Q},
    \end{aligned}
\end{equation*} where $C_{1}=\max\{5^{2/3}G^{2},4(G^{2}+\sigma^{2})+32(2G+Lr(\mathcal{K}))^{2},2.25(G^{2}+\sigma^{2})+7(2G+Lr(\mathcal{K}))^{2}/3\}$.

For each $j\in[N]$, we have 
\begin{equation*}
    \begin{aligned}
    & \sum_{t=1}^{T}\sum_{i=1}^{N}\E(|f_{t,i}(\bar{\x}(t))-f_{t,i}(\x_{j}(t))|) \\
    \le & GK\sum_{q=1}^{Q}\sum_{i=1}^{N}\|\x_{j}^{(K)}(q)-\bar{\x}^{(K)}(q)\| \\
    \le & GKN\sum_{q=1}^{Q}\sqrt{\sum_{i=1}^{N}\|\x_{i}^{(K)}(q)-\bar{\x}^{(K)}(q)\|^{2}} \\
    \le & \frac{GN^{3/2}r(\mathcal{K})Q}{1-\beta}.
    \end{aligned}
\end{equation*}
As a result, we have
\begin{equation*}
    \begin{aligned}
    & \sum_{t=1}^{T}\sum_{i=1}^{N}(1-1/e)f_{t,i}(\x^{*})-\sum_{t=1}^{T}\sum_{i=1}^{N}\E(f_{t,i}(\x_{j}(t)))\\
    \le & NG\mathrm{diam}(\mathcal{K})\frac{Q}{\gamma}+(\frac{LNr^{2}(\mathcal{K})}{2}+\frac{Nr(\mathcal{K})(L\mathrm{diam}(\mathcal{K})+GN^{1/2})}{1-\beta})Q \\
    & + N\mathrm{diam}(\mathcal{K})(2G+Lr(\mathcal{K}))\frac{Q\log(K+1)}{\gamma} + 2N\mathrm{diam}(\mathcal{K})\sqrt{C_{1}}QK^{2/3} \\
    & +\mathrm{diam}(\mathcal{K})N\sqrt{2(\sigma^{2}+G^{2})}\frac{KQ\gamma}{1-(1-\gamma)\beta}+M_{0}NK\sqrt{Q}.
    \end{aligned}
\end{equation*}

\end{proof}

%% file: arxiv_v1/Appendix2.tex
\subsection{The properties of Gradient Estimate for the Boosting Auxiliary Function}
In Section~\ref{sec:auxiliary}, for each monotone continous DR-submodular function $f:\mathcal{X}\rightarrow\R_{+}$, we propose an auxiliary function $F(\x)=\int_{0}^{1}\frac{e^{z-1}}{z}f(z*\x)\mathrm{d}z$. From the definition of the auxiliary function $F$, we know that
\begin{equation}\label{equ:333}
    \nabla F(\x)=\int_{0}^{1}e^{z-1}\nabla f(z*\x)\mathrm{d}z.
\end{equation}
As a result, it is not possible to directly compute the gradient of $F$ throughout Equation~(\ref{equ:333}). Instead, we present an estimate, i.e.,  $(1-1/e)\widetilde{\nabla}f(z*\x)$, where $z$ is a sample from the random variable $\mathbf{Z}\in[0,1]$ with Pr$(\mathbf{Z}\le z)=\frac{ e^{z-1}-1/e}{1-1/e}$. Next, we will  show some important properties about this estimate.
\begin{lemma}[\cite{zhang2022boosting}]\label{boost1} 
If $z$ is a sample of random variable $\mathbf{Z}$ and $\mathbb{E}(\widetilde{\nabla}f(\x)|\x)=\nabla f(\x)$, we have $(1-1/e)\mathbb{E}\left(\widetilde{\nabla}f(z*\x)|\x\right)=\nabla F(\x)$.
\end{lemma}
\begin{proof}
From $\mathbb{E}(\widetilde{\nabla}f(\x)|\x)=\nabla f(\x)$, we know that $\mathbb{E}\left(\widetilde{\nabla}f(z*\x)|z,\x\right)=\nabla f(z*\x)$. As a result, 
\begin{equation*}
    \begin{aligned}
     & (1-1/e)\mathbb{E}\left(\widetilde{\nabla}f(z*\x)|\x\right)\\
     = & (1-1/e)\int_{0}^{1}\frac{e^{z-1}}{1-1/e}\mathbb{E}\left(\widetilde{\nabla}f(z*\x)|z,\x\right)\mathrm{d}z\\
     = & \int_{0}^{1}e^{z-1}\nabla f(z*\x)\mathrm{d}z\\
     = & \nabla F(\x).
    \end{aligned}
\end{equation*}
\end{proof}

\subsection{The Proof of Theorem 2}
In this section, we assume $F_{t,i}$ is the boosting auxiliary function for each local monotone continuous DR-submodular function $f_{t,i}$, i.e., $F_{t,i}=\int_{0}^{1}\frac{e^{z-1}}{z}f_{t,i}(z*\x)dz$. Before going into the detail, we introduce some notations, i.e., 
\begin{equation*}
    \begin{aligned}
    &\bar{\x}(t)=\frac{\sum_{i=1}^{N}\x_{i}(t)}{N},\\
    &\bar{\y}(t)=\frac{\sum_{i=1}^{N}\y_{i}(t)}{N},\\
    &\mathbf{r}_{i}(t)=\x_{i}(t)-\y_{i}(t),\\
    &\bar{\mathbf{r}}(t)=\frac{\sum_{i=1}^{N}\mathbf{r}_{i}(t)}{N}.\\
    \end{aligned}
\end{equation*}
First, we prove that
\begin{lemma}\label{lemma:thm2.1}
If $\|\widetilde{\nabla}f_{t,i}(\x)\|\le G_{1}$ for any $i\in[N]$,$t\in[T]$ and $\x\in\mathcal{X}$, we could conclude that $\|\mathbf{r}_{i}(t+1)\|\le\eta_{t}(1-1/e)G_{1}$ for any $t\in[T]$.
\end{lemma}

\begin{proof}
\begin{equation*}
    \begin{aligned}
    & \|\mathbf{r}_{i}(t+1)\|\\
    = & \|\x_{i}(t+1)-\y_{i}(t+1)\|\\
    \le & \|\sum_{j\in\mathcal{N}_{i}\cup\{i\}}a_{ij}\x_{j}(t)-\y_{i}(t+1)\|\\
    \le & \|\eta_{t}(1-1/e)\widetilde{\nabla}f_{t,i}(z_{i}(t)*\x_{i}(t))\|\\
    \le & \eta_{t}(1-1/e)G_{1},
    \end{aligned}
\end{equation*}
where the first inequality follows from $\sum_{j\in\mathcal{N}_{i}\cup\{i\}}a_{ij}\x_{j}(t)\in\mathcal{K}$ and $\x_{i}(t+1)=\arg\min_{\mathbf{z}\in\mathcal{K}}\|\mathbf{z}-\y_{i}(t+1)\|$.
\end{proof}

\begin{lemma}\label{lemma:thm2.3}
If $\|\widetilde{\nabla}f_{t,i}(\x)\|\le G$ for any $i\in[N]$,$t\in[T]$ and $\x\in\mathcal{X}$, we have
\begin{equation*}
    \begin{aligned}
     & \sum_{i=1}^{N}\E(f_{t,i}(\x_{i}(t)))-(1-\frac{1}{e})\sum_{i=1}^{N}f_{t,i}(\x^{*})\\
     \ge & \frac{N}{2}\frac{\E(\|\bar{\x}(t+1)-\x^{*}\|^{2})-\E(\|\bar{\x}(t)-\x^{*}\|^{2})}{\eta_{t}}-2\eta_{t}(1-1/e)^{2}NG_{1}\\
     &-(1-1/e)G_{1}(\sum_{i=1}^{N}\E(\|\bar{\x}(t)-\y_{i}(t+1)\|)+\sum_{i=1}^{N}\E(\|\bar{\x}(t)-\x_{i}(t)\|)).
    \end{aligned}
\end{equation*}
\end{lemma}
\begin{proof}
First, we derive a equality about the $\bar{\y}_{t+1}$, namely, 
\begin{equation}\label{equ:2.3.0}
    \begin{aligned}
    \bar{\y}(t+1)=&\frac{\sum_{i=1}^{N}\y_{i}(t+1)}{N}\\
    =&\frac{1}{N}\sum_{i=1}^{N}\sum_{j\in\mathcal{N}_{i}\cup\{i\}}(a_{ij}\x_{j}(t)+\eta_{t}(1-1/e)\widetilde{\nabla}f_{t,i}(z_{i}(t)*\x_{i}(t)))\\
    =&\frac{1}{N}\sum_{j=1}^{N}\sum_{i\in\mathcal{N}_{j}\cup\{j\}}(a_{ij}\x_{j}(t))+\frac{(1-1/e)\eta_{t}}{N}\sum_{i=1}^{N}\widetilde{\nabla}f_{t,i}(z_{i}(t)*\x_{i}(t))\\
    =&\frac{1}{N}\sum_{j=1}^{N}\x_{j}(t))+\frac{(1-1/e)\eta_{t}}{N}\sum_{i=1}^{N}\widetilde{\nabla}f_{t,i}(z_{i}(t)*\x_{i}(t))\\
     =&\bar{\x}(t)+\frac{(1-1/e)\eta_{t}}{N}\sum_{i=1}^{N}\widetilde{\nabla}f_{t,i}(z_{i}(t)*\x_{i}(t)).
    \end{aligned}
\end{equation}
From the Equation~\ref{equ:2.3.0}, we have
\begin{equation}\label{equ:2.3.1}
    \begin{aligned}
    \bar{\x}(t+1)&=\bar{\x}(t+1)-\bar{\y}(t+1)+\bar{\y}(t+1)\\
    &=\bar{\mathbf{r}}(t+1)+\bar{\x}(t)+\frac{(1-1/e)\eta_{t}}{N}\sum_{i=1}^{N}\widetilde{\nabla}f_{t,i}(z_{i}(t)*\x_{i}(t)),
    \end{aligned}
\end{equation}
where the second equality follows from Equation~(\ref{equ:2.3.0}) and the definition of $\bar{\mathbf{r}}(t+1)$.

From Equation~(\ref{equ:2.3.1}), we prove that
\begin{equation}\label{equ:2.3.2}
    \begin{aligned}
    & \|\bar{\x}(t+1)-\x^{*}\|^{2}\\
    = & \|\bar{\x}(t)-\x^{*}+\frac{(1-1/e)\eta_{t}}{N}\sum_{i=1}^{N}\widetilde{\nabla}f_{t,i}(z_{i}(t)*\x_{i}(t))+\bar{\mathbf{r}}(t+1)\|^{2}\\
    = & \|\bar{\x}(t)-\x^{*}\|^{2}+2\langle\bar{\x}(t)-\x^{*},\frac{\eta_{t}}{N}\sum_{i=1}^{N}(1-1/e)\widetilde{\nabla}f_{t,i}(z_{i}(t)*\x_{i}(t))+\bar{\mathbf{r}}(t+1)\rangle\\
    & + \|\frac{(1-1/e)\eta_{t}}{N}\sum_{i=1}^{N}\widetilde{\nabla}f_{t,i}(z_{i}(t)*\x_{i}(t))+\bar{\mathbf{r}}(t+1)\|^{2}.
    \end{aligned}
\end{equation}
First, we derive the upper bound of the $\langle\bar{\x}(t)-\x^{*},\bar{\mathbf{r}}(t+1)\rangle$, i.e., 
\begin{equation}\label{equ:2.3.3}
    \begin{aligned}
    &\langle\bar{\x}(t)-\x^{*},\bar{\mathbf{r}}(t+1)\rangle\\
    = & \frac{1}{N}\sum_{i=1}^{N}\langle\bar{\x}(t)-\x^{*},\mathbf{r}_{i}(t+1)\rangle\\
    = & \frac{1}{N}\sum_{i=1}^{N}\langle\bar{\x}(t)-\y_{i}(t+1)+\y_{i}(t+1)-\x_{i}(t+1)+\x_{i}(t+1)-\x^{*},\mathbf{r}_{i}(t+1)\rangle\\
    = & \frac{1}{N}\sum_{i=1}^{N}\langle\bar{\x}(t)-\y_{i}(t+1)-\mathbf{r}_{i}(t+1)+\x_{i}(t+1)-\x^{*},\mathbf{r}_{i}(t+1)\rangle\\
    \le &\frac{1}{N}\sum_{i=1}^{N}\langle\bar{\x}(t)-\y_{i}(t+1),\mathbf{r}_{i}(t+1)\rangle+\langle \x_{i}(t+1)-\x^{*},\mathbf{r}_{i}(t+1)\rangle\\
    \le &\frac{1}{N}\sum_{i=1}^{N}\langle\bar{\x}(t)-\y_{i}(t+1),\mathbf{r}_{i}(t+1)\rangle\\
    \le &\frac{1}{N}\sum_{i=1}^{N}\|\mathbf{r}_{i}(t+1)\|\|\bar{\x}(t)-\y_{i}(t+1)\|\\
    \le &\frac{\eta_{t}(1-1/e)G_{1}}{N}\|\bar{\x}(t)-\y_{i}(t+1)\|,
     \end{aligned}
\end{equation} where the second inequality comes from $\langle \x_{i}(t+1)-\x^{*},\mathbf{r}_{i}(t+1)\rangle\le0$; the final inequality follows from Lemma~\ref{lemma:thm2.1}.

Then, we derive the  the upper bound of the expectation of $\langle\bar{\x}(t)-\x^{*},\frac{\eta_{t}}{N}\sum_{i=1}^{N}(1-1/e)\widetilde{\nabla}f_{t,i}(z_{i}(t)*\x_{i}(t))
\rangle$, namely, 
\begin{equation}\label{equ:2.3.4}
    \begin{aligned}
    &\frac{(1-1/e)\eta_{t}}{N}\sum_{i=1}^{N}\E(\langle\bar{\x}(t)-\x^{*},\widetilde{\nabla}f_{t,i}(z_{i}(t)*\x_{i}(t))\rangle)\\
    = &\frac{(1-1/e)\eta_{t}}{N}\sum_{i=1}^{N}\E(\langle\bar{\x}(t)-\x_{i}(t)+\x_{i}(t)-\x^{*},\widetilde{\nabla}f_{t,i}(z_{i}(t)*\x_{i}(t))\rangle)\\
     = &\frac{(1-1/e)\eta_{t}}{N}\sum_{i=1}^{N}\E(\langle\bar{\x}(t)-\x_{i}(t),\widetilde{\nabla}f_{t,i}(z_{i}(t)*\x_{i}(t))\rangle)\\
    &+\frac{\eta_{t}}{N}\sum_{i=1}^{N}\E(\langle \x_{i}(t)-\x^{*},\E(\widetilde{\nabla}(1-1/e)f_{t,i}(z_{i}(t)*\x_{i}(t))|\x_{i}(t)))\rangle)\\
    \le &\frac{(1-1/e)\eta_{t}}{N}\sum_{i=1}^{N}G_{1}\E(\|\bar{\x}(t)-\x_{i}(t)\|)+\frac{\eta_{t}}{N}\sum_{i=1}^{N}\E(\langle \x_{i}(t)-\x^{*},\nabla F_{t,i}(\x_{i}(t))\rangle)\\
    \le &\frac{(1-1/e)\eta_{t}}{N}\sum_{i=1}^{N}G_{1}\E(\|\bar{\x}(t)-\x_{i}(t)\|)+\frac{\eta_{t}}{N}\sum_{i=1}^{N}(f_{t,i}(\x_{i}(t))-(1-\frac{1}{e})f_{t,i}(\x^{*})),
     \end{aligned}
\end{equation}
where the first inequality comes from Lemma~\ref{boost1}; the final inequality follows from Lemma~\ref{lemma:bianfen}.

Finally, we prove that
\begin{equation}\label{equ:2.3.5}
    \begin{aligned}
    & \|\frac{(1-1/e)\eta_{t}}{N}\sum_{i=1}^{N}\widetilde{\nabla}f_{t,i}(z_{i}(t)*\x_{i}(t))+\bar{\mathbf{r}}(t+1)\|^{2}
    \\
    \le & 2(\|\frac{(1-1/e)\eta_{t}}{N}\sum_{i=1}^{N}\widetilde{\nabla}f_{t,i}(z_{i}(t)*\x_{i}(t))\|^{2}+\|\bar{\mathbf{r}}(t+1)\|^{2})\\
    \le & 4(1-1/e)^{2}\eta_{t}^{2}G_{1}^{2},
    \end{aligned}
\end{equation} where the first inequality comes from the Cauchy–Schwarz inequality; the final inequality comes from Lemma~\ref{lemma:thm2.1}. After merging Equation~(\ref{equ:2.3.3})-(\ref{equ:2.3.5}) to Equation~(\ref{equ:2.3.2}), we get 
\begin{equation*}
    \begin{aligned}
    & \E(\|\bar{\x}(t+1)-\x^{*}\|^{2})\\
    \le & \E(\|\bar{\x}(t)-\x^{*}\|^{2})+4(1-1/e)^{2}\eta_{t}^{2}G_{1}^{2}+\frac{2(1-1/e)\eta_{t}G_{1}}{N}\sum_{i=1}^{N}\E(\|\bar{\x}(t)-\y_{i}(t+1)\|)\\
    & + \frac{2(1-1/e)\eta_{t}}{N}\sum_{i=1}^{N}G_{1}\E(\|\bar{\x}(t)-\x_{i}(t)\|)+\frac{2\eta_{t}}{N}\sum_{i=1}^{N}(f_{t,i}(\x_{i}(t))-(1-\frac{1}{e})f_{t,i}(\x^{*})).
    \end{aligned}
\end{equation*}
As a result, 
\begin{equation*}
    \begin{aligned}
     & \sum_{i=1}^{N}\E(f_{t,i}(\x_{i}(t)))-(1-\frac{1}{e})\sum_{i=1}^{N}\E(f_{t,i}(\x^{*}))\\
     \ge & \frac{N}{2}\frac{\E(\|\bar{\x}(t+1)-\x^{*}\|^{2})-\E(\|\bar{\x}(t)-\x^{*}\|^{2})}{\eta_{t}}-2\eta_{t}(1-1/e)^{2}NG_{1}^{2}\\
     & - (1-1/e)G_{1}(\sum_{i=1}^{N}\E(\|\bar{\x}(t)-\y_{i}(t+1)\|)+\sum_{i=1}^{N}\E(\|\bar{\x}(t)-\x_{i}(t)\|)).
    \end{aligned}
\end{equation*}
\end{proof}
Next, we derive an upper bound on the sum of the distances between the local iterates $\x_{i}(t)$~(or $\y_{i}(t+1)$) and the average vector $\bar{\x}(t)$. We first recall the definition of the second largest magnitude of the eigenvalues of matrix $\A$, i.e., 
\begin{definition}\label{def3}
Consider the eigenvalues of $\A$ which can be sorted in a non-increasing order as $1=\lambda_{1}(\A)\ge\lambda_{2}(\A)\dots\ge\lambda_{n}(\A)>-1$. Define $\beta$ as the second largest magnitude of the eigenvalues of $\A$, i.e., $\beta=\max(|\lambda_{2}(\A)|,|\lambda_{n}(\A)|)$.
\end{definition}
\begin{lemma}\label{lemma:thm2.4}
If $\|\widetilde{\nabla}f_{t,i}(\x)\|\le G_{1}$ for any $i\in[N]$,$t\in[T]$ and $\x\in\mathcal{X}$, we have
\begin{equation*}
    \begin{aligned}
    &\sqrt{\sum_{i=1}^{N}\|\bar{\x}(t+1)-\x_{i}(t+1)\|^{2}}\le2\sum_{k=1}^{t}\beta^{t-k}\sqrt{N}\eta_{k}G_{1},\\
   &\sqrt{\sum_{i=1}^{N}\|\bar{\x}(t)-\y_{i}(t+1)\|^{2}}\le2\sum_{k=1}^{t}\beta^{t-k}\sqrt{N}\eta_{k}G_{1}.
    \end{aligned}
\end{equation*}
\end{lemma}
\begin{proof}
We first define some auxiliary vectors throughout the concatenation of the local variables. We set:
\begin{equation*}
    \begin{aligned}
       &\y(t+1)=[\y_{1}(t+1);\dots;\y_{N}(t+1)]\in\R^{Nn},\\ 
       &\x(t+1)=[\x_{1}(t+1);\dots;\x_{N}(t+1)]\in\R^{Nn},\\ 
       &\mathbf{r}(t+1)=[\mathbf{r}_{1}(t+1);\dots;\mathbf{r}_{N}(t+1)]\in\R^{Nn},\\
       &\mathbf{g}(t)=(1-1/e)[\widetilde{\nabla}f_{t,1}(z_{1}(t)*\x_{1}(t));\dots;\widetilde{\nabla}f_{t,N}(z_{N}(t)*\x_{N}(t)]\in\R^{Nn}.\\
    \end{aligned}
   \end{equation*}
According to Algorithm~\ref{alg:2}, we have
\begin{equation*}
    \begin{aligned}
       \y(t+1)&=(\A\otimes \mathbf{I})\x(t)+\eta_{t}\mathbf{g}(t),\\
       \x(t+1)&=\mathbf{r}(t+1)+\y(t+1)=\mathbf{r}(t+1)+(\A\otimes\mathbf{I})\x(t)+\eta_{t}\mathbf{g}(t),\\
       \end{aligned}
\end{equation*} where the $\mathbf{I}$ is the $n$-dimensional identity matrix.
By iteration, we have
\begin{equation*}
    \begin{aligned}
       \x(t+1)&=\mathbf{r}(t+1)+(\A\otimes\mathbf{I})\x(t)+\eta_{t}\mathbf{g}(t)\\
       &=\sum_{k=1}^{t}(\A\otimes \mathbf{I})^{t-k}\mathbf{r}(k+1)+\sum_{k=1}^{t}(\A\otimes \mathbf{I})^{t-k}\eta_{k}\mathbf{g}(k).
       \end{aligned}
\end{equation*}
Therefore, we also have
\begin{equation*}
    \begin{aligned}
    \y(t+1)=(\A\otimes \mathbf{I})\x(t)+\eta_{t}\mathbf{g}(t)
    =\sum_{k=1}^{t-1}(\A\otimes \mathbf{I})^{t-k}\mathbf{r}(k+1)+\sum_{k=1}^{t}(\A\otimes \mathbf{I})^{t-k}\eta_{k}\mathbf{g}(k).
    \end{aligned}
\end{equation*}
Then, we construct a new vector.
\begin{equation*}
    \begin{aligned}
    (\frac{\mathbf{1}\mathbf{1}^{T}}{N}\otimes \mathbf{I})\x(t+1)&=[\bar{\x}(t+1);\dots;\bar{\x}(t+1)]\\
    &=\sum_{k=1}^{t}(\frac{\mathbf{1}\mathbf{1}^{T}}{N}\otimes \mathbf{I})\mathbf{r}(k+1)+\sum_{k=1}^{t}(\frac{\mathbf{1}\mathbf{1}^{T}}{N}\otimes I)\eta_{k}\mathbf{g}(k),
    \end{aligned}
\end{equation*} 
where $\mathbf{1}$ is a $N$-dimensional column vector whose all elements is $1$; the final equality follows from $\mathbf{1}^{T}A=\mathbf{1}^{T}$.

Therefore, we have
\begin{equation*}
    \begin{aligned}
    & \sqrt{\sum_{i=1}^{N}\|\bar{\x}(t+1)-\x_{i}(t+1)\|^{2}}\\
    = & \|(\frac{\mathbf{1}\mathbf{1}^{T}}{N}\otimes \mathbf{I})\x(t+1)-\x(t+1)\|\\
    = & \|\sum_{k=1}^{t}((\frac{\mathbf{1}\mathbf{1}^{T}}{N}-\A^{t-k})\otimes\mathbf{I})\mathbf{r}(k+1)+\sum_{k=1}^{t}((\frac{\mathbf{1}\mathbf{1}^{T}}{N}-\A^{t-k})\otimes\mathbf{I})\eta_{k}\mathbf{g}(k)\|\\
    \le & \sum_{k=1}^{t}\|\frac{\mathbf{1}\mathbf{1}^{T}}{N}-\A^{t-k}\|\|\mathbf{r}(k+1)\|+\sum_{k=1}^{t}\|\frac{\mathbf{1}\mathbf{1}^{T}}{N}-\A^{t-k}\|\|\eta_{k}\mathbf{g}(k)\|\\
    \le & 2\sum_{k=1}^{t}\beta^{t-k}\sqrt{N}\eta_{k}(1-1/e)G_{1},
    \end{aligned}
\end{equation*} 
where the final inequality comes from $\|\frac{\mathbf{1}\mathbf{1}^{T}}{N}-\A^{t-k}\|\le\beta^{t-k}$ for any $k\in\{1,\dots,t\}$ and Lemma~\ref{lemma:thm2.1}.

Similarly, we have
\begin{equation*}
    \begin{aligned}
    & \sqrt{\sum_{i=1}^{N}\|\bar{\x}(t)-\y_{i}(t+1)\|^{2}}\\
    = & \|(\frac{\mathbf{1}\mathbf{1}^{T}}{N}\otimes \mathbf{I})\x(t)-\y(t+1)\|\\
    = & \|\sum_{k=1}^{t-1}((\frac{\mathbf{1}\mathbf{1}^{T}}{N}-\A^{t-k})\otimes I)\mathbf{r}(k+1)+\sum_{k=1}^{t-1}((\frac{\mathbf{1}\mathbf{1}^{T}}{N}-\A^{t-k})\otimes I)\eta_{k}\mathbf{g}(k)-\eta_{t}\mathbf{g}(t)\|\\
    \le & \sum_{k=1}^{t-1}\|\frac{\mathbf{1}\mathbf{1}^{T}}{N}-\A^{t-k}\|\|\mathbf{r}(k+1)\|+\sum_{k=1}^{t-1}\|\frac{\mathbf{1}\mathbf{1}^{T}}{N}-\A^{t-k}\|\|\eta_{k}\mathbf{g}(k)\|+\|\eta_{t}\mathbf{g}(t)\|\\
    \le & 2\sum_{k=1}^{t}\beta^{t-k}\sqrt{N}\eta_{k}(1-1/e)G_{1}.
    \end{aligned}
\end{equation*}
\end{proof}
Now, we prove Theorem~\ref{thm:2}.
\begin{proof}
From Lemma~\ref{lemma:thm2.3}, we have
\begin{equation*}
    \begin{aligned}
     & \sum_{t=1}^{T}\sum_{i=1}^{N}\E(f_{t,i}(\x_{i}(t)))-(1-1/e)\sum_{t=1}^{T}\sum_{i=1}^{N}f_{t,i}(\x^{*})\\
     \ge & \frac{N}{2}\sum_{t=1}^{T}\frac{\E(\|\bar{\x}(t+1)-\x^{*}\|^{2})-\E(\|\bar{\x}(t)-\x^{*}\|^{2})}{\eta_{t}}-2(1-1/e)^{2}NG_{1}^{2}\sum_{t=1}^{T}\eta_{t}\\
     & - (1-1/e)G_{1}(\sum_{t=1}^{T}\sum_{i=1}^{N}\E(\|\bar{\x}(t)-\y_{i}(t+1)\|^{2})+\sum_{t=1}^{T}\sum_{i=1}^{N}\E(\|\bar{\x}(t)-\x_{i}(t)\|^{2}))\\
     \ge & \frac{N}{2}\sum_{t=1}^{T}\frac{\E(\|\bar{\x}(t+1)-\x^{*}\|^{2})-\E(\|\bar{\x}(t)-\x^{*}\|^{2})}{\eta_{t}}-2(1-1/e)^{2}NG_{1}^{2}\sum_{t=1}^{T}\eta_{t}\\
     & - (1-1/e)G_{1}\E(\sum_{t=1}^{T}\sqrt{N}\sqrt{\sum_{i=1}^{N}\|\bar{\x}(t)-\y_{i}(t+1)\|^{2}}+\sum_{t=1}^{T}\sqrt{N}\sqrt{\sum_{i=1}^{N}\|\bar{\x}(t)-\x_{i}(t)\|^{2}})\\
     \ge & \frac{N}{2}\sum_{t=1}^{T}\frac{\E(\|\bar{\x}(t+1)-\x^{*}\|^{2})-\E(\|\bar{\x}(t)-\x^{*}\|^{2})}{\eta_{t}}-2(1-1/e)^{2}NG_{1}^{2}\sum_{t=1}^{T}\eta_{t}\\ 
     & -4(1-1/e)^{2}NG_{1}^{2}(\sum_{t=1}^{T}\sum_{k=1}^{t}\beta^{t-k}\eta_{k})\\
     = & \frac{N}{2}\sum_{t=1}^{T}\frac{\E(\|\bar{\x}(t+1)-\x^{*}\|^{2})-\E(\|\bar{\x}(t)-\x^{*}\|^{2})}{\eta_{t}}-2(1-1/e)^{2}NG_{1}^{2}\sum_{t=1}^{T}\eta_{t}\\
     & -4(1-1/e)^{2}NG_{1}^{2}(\sum_{k=1}^{T}\sum_{t=k}^{T}\beta^{t-k}\eta_{k})\\
     \ge & \frac{N}{2}\sum_{t=1}^{T}\frac{\E(\|\bar{\x}(t+1)-\x^{*}\|^{2})-\E(\|\bar{\x}(t)-\x^{*}\|^{2})}{\eta_{t}}-2(1-1/e)^{2}NG_{1}^{2}\sum_{t=1}^{T}\eta_{t}\\
     & -\frac{4(1-1/e)^{2}NG_{1}^{2}}{1-\beta}\sum_{k=1}^{T}\eta_{k}.\\
\end{aligned}
\end{equation*}
Next, we show that
\begin{equation*}
    \begin{aligned}
    &\frac{N}{2}\sum_{t=1}^{T}\frac{\E(\|\bar{\x}(t+1)-\x^{*}\|^{2})-\E(\|\bar{\x}(t)-\x^{*}\|^{2})}{\eta_{t}}\\
    = & \frac{N}{2}(\frac{\E(\|\bar{\x}(T+1)-\x^{*}\|^{2})}{\eta_{T}}+\sum_{t=1}^{T-1}\E(\|\bar{\x}(t+1)-\x^{*}\|^{2})(\frac{1}{\eta_{t}}-\frac{1}{\eta_{t+1}})-\frac{\E(\|\bar{\x}(1)-\x^{*}\|^{2})}{\eta_{1}})\\
    \ge & \frac{N}{2}(\frac{\E(\|\bar{\x}(T+1)-\x^{*}\|^{2})}{\eta_{T}}+\mathrm{diam}(\K)\sum_{t=1}^{T-1}(\frac{1}{\eta_{t}}-\frac{1}{\eta_{t+1}})-\frac{\E(\|\bar{\x}(1)-\x^{*}\|^{2})}{\eta_{1}})\\
    \ge & \frac{N}{2}(\frac{\E(\|\bar{\x}(T+1)-\x^{*}\|^{2})}{\eta_{T}}+\frac{\mathrm{diam}(\K)}{\eta_{1}}-\frac{\mathrm{diam}(\K)}{\eta_{T}}-\frac{\|\bar{\x}(1)-\x^{*}\|^{2}}{\eta_{1}})\\
    \ge & -\frac{N}{2}\frac{\mathrm{diam}(\K)}{\eta_{T}}.
    \end{aligned}
\end{equation*}
Therefore, we obtain
\begin{equation*}
    \begin{aligned}
     &\sum_{t=1}^{T}\sum_{i=1}^{N}\E(f_{t,i}(\x_{i}(t)))-(1-1/e)\sum_{t=1}^{T}\sum_{i=1}^{N}f_{t,i}(\x^{*}) \\
     \ge &-\frac{N\mathrm{diam}(\K)}{2\eta_{T}}-(1-1/e)^{2}(2NG_{1}^{2}+\frac{4NG_{1}^{2}}{1-\beta})\sum_{t=1}^{T}\eta_{t}.
    \end{aligned}
\end{equation*}
Due to $\|\widetilde{\nabla}f_{t,i}\|\le G_{1}$, we have$\|\nabla f_{t,i}\|\le G_{1}$ such that $|f_{t,i}(\x_{i}(t))-f_{t,i}(\x_{j}(t))|\le|f_{t,i}(\x_{i}(t))-f_{t,i}(\bar{\x}(t))|+|f_{t,i}(\bar{\x}(t))-f_{t,i}(\x_{j}(t))|\le G_{1}\|\x_{i}(t)-\bar{\x}(t)\|+G_{1}\|\x_{j}(t)-\bar{\x}(t)\|$. As a result,
\begin{equation*}
    \begin{aligned}
     & \sum_{t=1}^{T}\sum_{i=1}^{N}|f_{t,i}(\x_{i}(t))-f_{t,i}(\x_{j}(t))|\\
     \le & G_{1}\sum_{t=1}^{T}\sum_{i=1}^{N}\|\x_{i}(t)-\bar{\x}(t)\|+G_{1}N\sum_{t=1}^{T}\|\x_{j}(t)-\bar{\x}(t)\|\\
     \le & G_{1}\sqrt{N}\sum_{t=1}^{T}\sqrt{\sum_{i=1}^{N}\|\x_{i}(t)-\bar{\x}(t)\|^{2}}+G_{1}N\sum_{t=1}^{T}\sqrt{\sum_{i=1}^{N}\|\x_{i}(t)-\bar{\x}(t)\|^{2}}\\
     \le & \frac{2G^{2}_{1}(N+N^{3/2})}{\beta}\sum_{t=1}^{T}\sum_{k=1}^{t}\beta^{t-k}\eta_{k}\\
     \le & \frac{2G^{2}_{1}(N+N^{3/2})}{\beta}\sum_{k=1}^{T}\sum_{t=k}^{T}\beta^{t-k}\eta_{k}\\
     \le & \frac{2G^{2}_{1}(N+N^{3/2})}{(1-\beta)\beta}\sum_{k=1}^{T}\eta_{k}.
    \end{aligned}
\end{equation*}
Finally, we obtain
\begin{equation*}
    \begin{aligned}
     &\sum_{t=1}^{T}\sum_{i=1}^{N}\E(f_{t,i}(\x_{j}(t)))-(1-1/e)\sum_{t=1}^{T}\sum_{i=1}^{N}f_{t,i}(\x^{*}) \\
     \ge &-\frac{N\mathrm{diam}(\K)}{2\eta_{T}}-(2NG_{1}^{2}+\frac{4NG_{1}^{2}}{1-\beta}+\frac{2G^{2}_{1}(N+N^{3/2})}{(1-\beta)\beta})\sum_{t=1}^{T}\eta_{t}.
    \end{aligned}
\end{equation*}
\end{proof}